\documentclass{scrartcl}

\usepackage{lmodern}
\usepackage{graphicx}
\usepackage{amsmath,amsfonts,amssymb,amsthm}
\usepackage{mathtools}
\usepackage{graphicx}
\usepackage{hyperref}
\usepackage{todonotes}
\usepackage{algorithmicx}
\usepackage{algorithm}
\usepackage{algpseudocode}

\usepackage{siunitx}
\usepackage{tabularray}
\usepackage[numbers,sort&compress]{natbib}

\newtheorem{thm}{Theorem}[section]
\newtheorem{lem}[thm]{Lemma}
\newtheorem{cor}[thm]{Corollary}

\theoremstyle{definition}
\newtheorem{definition}{Definition}
\theoremstyle{remark}
\newtheorem{rmk}[thm]{Remark}
\newtheorem{example}[thm]{Example}

\renewcommand{\P}{\mathbb{P}}
\newcommand{\R}{\mathbb{R}}
\renewcommand{\H}{\mathbb{H}}
\newcommand{\D}{\mathbb{D}}
\renewcommand{\DH}{\D\H}
\newcommand{\eps}{\varepsilon}
\newcommand{\qi}{\mathbf{i}}
\newcommand{\qj}{\mathbf{j}}
\newcommand{\qk}{\mathbf{k}}
\newcommand{\cj}[1]{{#1}^\ast}
\newcommand{\ej}[1]{{#1}_\eps}
\newcommand{\No}[1]{#1\cj{#1}}
\newcommand{\SE}[1][3]{\operatorname{SE}(#1)}

\newcommand{\bll}{\boldsymbol{\ell}}

\DeclareMathOperator{\Scal}{Scal}
\DeclareMathOperator{\Vect}{Vect}
\DeclareMathOperator{\bdeg}{bdeg}

\hypersetup{
  pdfauthor={Johanna Frischauf, Martin Pfurner, Daniel F. Scharler, Hans-Peter Schröcker},pdftitle={A Multi-Bennett 8R Mechanism Obtained From Factorization of Bivariate Motion Polynomials},}

\setkomafont{title}{\normalfont\bfseries}
\setkomafont{sectioning}{\normalfont\bfseries}

\title{A Multi-Bennett 8R Mechanism Obtained From Factorization of Bivariate Motion Polynomials}
\author{
  Johanna Frischauf, Martin Pfurner,\\
  Daniel F. Scharler, Hans-Peter Schröcker\\[1ex]
  Department of Basic Sciences in Engineering Sciences\\
  University of Innsbruck, Austria
}

\begin{document}

\maketitle

\begin{abstract}
    We present a closed-loop 8R mechanism with two degrees of freedom whose motion
  exhibits curious properties. In any point of a two-dimensional component of
  its configuration variety it is possible to fix every second joint while
  retaining one degree of freedom. This shows that the even and the odd axes,
  respectively, always form a Bennett mechanism. In this mechanism, opposite
  distances and angles are equal and all offsets are zero. The 8R mechanism has
  four ``totally aligned'' configurations in which the common normals of any
  pair of consecutive axes coincide.
 \end{abstract}

\section{Introduction}
\label{sec:introduction}

Overconstrained linkages is a long-lasting but still highly
active topic of research in mechanism science. For several decades, researchers
focused on overconstrained mechanisms consisting of a single loop of $n \le 6$
revolute joints (R), prismatic joints (P), or, sometimes, helical joints (H).
New linkages of that type are continuously being discovered, often by craftily
combining known linkages \cite{wohlhart91, chen07, baker05}, sometimes via novel concepts
for their construction. One of these concepts is the
factorization of motion polynomials \cite{hegedus13}. It gave rise to the
construction of the only class of overconstrained 6R linkages with still unknown
relations between its Denavit-Hartenberg parameters. In
\cite{hegedus15,gallet17,li18,li20,liu21,liu21b}, motion polynomial
factorization was exploited for the synthesis of linkages.

In spite of some attempts, a complete classification of overconstrained
single-loop linkages is currently out of reach. It is thus natural that research
efforts shifted towards the investigation of single-loop linkages consisting of
$n \ge 7$ links with, generically, $n - 6 \ge 1$ degrees of freedom. (The
classification of single-loop linkages with $n \ge 7$ links and more than $n -
6$ degrees of freedom has recently been completed in \cite{guerreiro22}.) A
guiding principle for their construction is existence of ``interesting''
properties of the mechanism's motion or its configuration variety. One example
is \cite{kong15}, where 7R linkages whose configuration variety contains
irreducible components of different dimensions -- a property that has been named
\emph{kinematotropic} in \cite{wohlhart96} -- are constructed. \cite{pfurner14}
combines mobile 4R linkages (Bennett linkages) or RPRP linkages to loops of 7R/P
joints whose configuration variety is reducible. The motion of the original 4R/P
linkages is obtained by locking of joints in certain configurations.
Analogically \cite{pfurner18} restricted a specially designed single-loop 8R
mechanism to its possible sub-motions. \cite{liu21} and \cite{liu21b} pursue
similar aims but use motion polynomial factorization techniques. Joint locking
is also used in \cite{kong16} for restricting a mechanism to a certain
subvariety of its total configuration space, although for a class of parallel
mechanisms.

Our contribution in this article is of similar spirit as the works cited above
but also differs in several aspects. We present an 8R linkage with two degrees
of freedom that has the weird property that it retains one degree of freedom
when simultaneously locking every second joint in \emph{any} configuration of a
two-dimensional subvariety of its total configuration space. This unique
property immediately implies that the quadruples of ``even'' or ``odd'' axes
form respective Bennett mechanisms in any configuration because Bennett
mechanisms constitute the only class of mobile spatial closed-loop linkages with
four revolute axes \cite{karger94}. We therefore refer to this mechanism by the
name ``multi-Bennett 8R mechanism''. Our aim in this paper is to prove these
properties by an algebraic construction and use this to derive some geometric
and kinematic characteristics of the thus obtained mechanism.

While combination of Bennett linkages is a common technique in this area
\cite{goldberg43,waldron68,baker93,pfurner14,kong15}, our example seems to be
novel. It is not geometrically motivated -- at least in the current state of our
understanding -- but rather based on an algebraic construction. As suggested by
examples in \cite{lercher21}, there exist \emph{bivariate motion polynomials}
that admit, in a non-trivial way, two factorizations into products of linear
\emph{univariate} factors with alternating indeterminates. These two times four
factors give rise to the revolute axes of the 8R mechanism and describe their
relative motions. The underlying bivariate factorization theory is currently
being explored \cite{lercher22:_bi-degree,lercher21} and is considerably harder
than in the univariate case. This is witnessed by our proof of existence in
Theorem~\ref{thm:primalpart}.

In spite of its algebraic construction, the 8R linkage is subject to severe
geometric constraints. We demonstrate this by computing simple necessary
relations between its Denavit-Hartenberg parameters in
Theorem~\ref{th:DH-relations}. In Theorems~\ref{th:folded} and
\ref{th:folded-Bennett} we describe remarkable properties of several discrete
configurations. Our algebraic approach is efficient for proving existence of
multi-Bennett 8R mechanisms and some aspects of its geometry. A complete
geometric characterization, which exists for the vast majority of comparable
mechanisms, can not be obtained in this way and is probably rather difficult.

We feel that its numerous special properties (simple Denavit-Hartenberg
parameters, two degrees of freedom that are easy to control via low-degree
rational parametrization, existence of special configurations) make our
mechanism a promising candidate for yet to be explored applications.

We continue this text by recalling some basic facts about motion polynomials and
their relation to mechanism science in Section~\ref{sec:preliminaries}. In
Section~\ref{sec:bivariate} we provide a proof for existence of
quaternion polynomials with two non-trivial univariate factorizations. The proof
is constructive and provides a good method to directly compute the underlying 8R
linkage. Nonetheless, we found the procedure insufficient for obtaining results
that are suitable for further processing and in particular for the computation
of Denavit-Hartenberg parameters. Thus, our further investigation of the
multi-Bennett 8R mechanism in Section~\ref{sec:8R-Mechanism} is based on
carefully selected coordinate frames and configurations. This simplification
results in formulas that are tractable by means of computer algebra and,
ultimately, provides the desired necessary relations among the
Denavit-Hartenberg parameters (Theorem~\ref{th:DH-relations}).

This paper is a continuation of \cite{lercher22:_8R-mechanism}, a conference
paper which verifies most of the claims made in this article at hand of a
concrete numeric example. Strict mathematical proofs of the claimed facts are
presented herein for the first time.

\section{Preliminaries}
\label{sec:preliminaries}

Our construction of the multi-Bennett mechanism is based on certain
factorizations of bivariate quaternion polynomials. In this section we provide a
brief introduction to some fundamental concepts that will be used later in this
text and we settle our notation.

Denote by $\H$ the four-dimensional associative real algebra of (real) quaternions. It
is generated by basis elements $\qi$, $\qj$, and $\qk$ via the relations
\begin{equation*}
  \qi^2 = \qj^2 = \qk^2 = \qi\qj\qk = -1.
\end{equation*}
A quaternion $h \in \H$ can be written as $h = h_0 + h_1\qi + h_2\qj + h_3\qk$
with real numbers $h_0$, $h_1$, $h_2$, $h_3$. Extending real scalars to dual
numbers $a + b \eps$ with $a$, $b \in \R$ and $\eps^2 = 0$ yields the algebra
$\DH$ of dual quaternions
\begin{equation*}
  \DH = \{h = h_p + \eps h_d \mid h_p, h_d \in \H \}.
\end{equation*}

The quaternions $h_p$ and $h_d$ are called \emph{primal part} and \emph{dual part} of $h$. Any quaternion can be viewed as a dual quaternion with vanishing dual part. We will therefore sometimes use the same symbol $h$ for real ($h \in \H$) and dual ($h \in \DH$) quaternions.

The conjugate dual quaternion $\cj{h}$ is obtained by
replacing $\qi$, $\qj$, and $\qk$ with $-\qi$, $-\qj$, and $-\qk$, respectively,
the $\eps$-conjugate $\ej{h}$ of a dual quaternion is obtained by replacing
$\eps$ with~$-\eps$.

Given a dual quaternion $h = h_p+\eps h_d$, where $h_p = h_0 + h_1 \qi + h_2\qj + h_3\qk$ and $h_d=h_4 + h_5\qi + h_6\qj + h_7\qk$, the value $\Scal(h) \coloneqq \frac{1}{2}(h + \cj{h}) =
h_0 + \eps h_4$ is called the \emph{scalar part} and $\Vect(h)
\coloneqq \frac{1}{2}(h - \cj{h}) = h_1\qi + h_2\qj + h_3\qk + \eps(h_5\qi +
h_6\qj + h_7\qk)$ the \emph{vector part} of~$h$. The dual quaternion norm of $h$ is $h\cj{h}$. It is the dual number
\begin{equation}
  \label{eq:1}
  h\cj{h} = h_p\cj{h_p} + \eps(h_p\cj{h_d} + h_d\cj{h_p}) = h_0^2 + h_1^2 + h_2^2 + h_3^2 + 2\eps(h_0h_4 + h_1h_5 + h_2h_6 + h_3h_7) \in \D.
\end{equation}
Dual quaternions satisfying $h\cj{h} = 1$ are said to be \emph{normalized} or
\emph{unit.} In this case, the dual part in \eqref{eq:1} vanishes, that is
\begin{equation}
  \label{eq:2}
  h_0h_4 + h_1h_5 + h_2h_6 + h_3h_7 = 0.
\end{equation}
This is well-known under the name \emph{Study condition.}

The dual quaternion $h = h_p + \eps h_d$ is
invertible if and only if $h_p \neq 0$. In this case, we have
\begin{equation*}
    h^{-1} = h_p^{-1} (1 - \eps h_dh_p^{-1}), \quad \text{where} \quad h_p^{-1} = \frac{\cj{h_p}}{h_p\cj{h_p}}.
\end{equation*}
If $h$ is unit, then $h^{-1} = \cj{h}$.

The multiplicative sub-group $\DH^\times \coloneqq \{h \in \DH \mid h\cj{h} \in
\R \setminus \{0\}\}$ modulo the real multiplicative group $\R^\times$ is
isomorphic to $\SE$, the group of rigid body displacements. Using homogeneous
coordinates in the projective space $\P^3(\R) = \P(\langle 1, \eps\qi, \eps\qj,
\eps\qk \rangle)$, the action of $h \in \DH^\times$ on $x = x_0 + \eps(x_1\qi +
x_2\qj + x_3\qk) \in \P^3(\R)$ is given by
\begin{equation}
  \label{eq:3}
  x \mapsto \ej{h} x \cj{h}.
\end{equation}

\subsection{Dual Quaternions and Line Geometry}
\label{sec:line-geometry}

In this paper, rotations around a fixed axis but with variable rotation angle
will play an important role. We therefore have a closer look at the
representation of straight lines (revolute axes) and rotations within the
framework of dual quaternions. Identifying the oriented revolute axis with
normalized Plücker coordinates $(p_1,p_2,p_3,q_1,q_2,q_3)$ in the sense of
\cite[Section~2]{pottmann10} with the unit dual quaternion $r = p_1\qi + p_2\qj
+ p_3\qk + \eps(q_1\qi + q_2\qj + q_3\qk)$, the rotation with angle $\varphi$
around $r$ is given by the unit dual quaternion
\begin{equation}
  \label{eq:4}
  h \coloneqq \cos(\tfrac{\varphi}{2}) + \sin(\tfrac{\varphi}{2})r,\quad \varphi \in [0, 2\pi),
\end{equation}
or, because we use homogeneous coordinates, by any of its non-zero real
multiples. Note that the dual quaternion $h$ of \eqref{eq:4} satisfies the Study
condition \eqref{eq:2} because $r$ satisfies the Plücker condition $r\cj{r} \in
\R \setminus \{0\}$.

The action \eqref{eq:3} on points can be used to transform straight lines by
transforming points on them. Points on a straight line given by its Plücker
coordinates can be found, for example, by \cite[Equation~(2.4)]{pottmann10}. A
straightforward calculation also provides us with a direct formula for
displacing a straight line $\ell$ whose Plücker coordinates are given as
vectorial dual quaternions:
\begin{equation}
  \label{eq:5}
  \ell \mapsto \ej{(h \ell \cj{h})}.
\end{equation}

\subsection{Dual Quaternion Polynomials}
\label{sec:polynomials}

The representation \eqref{eq:4} of a rotation around an oriented general axis is
only unique up to multiplication with a real scalar. Assuming, for the time
being, $\varphi \neq 0$, we can divide \eqref{eq:4} by $\sin\frac{\varphi}{2}$,
substitute $\cot\frac{\varphi}{2}$ with $-t$ and multiply the result with $-1$
to see that the linear dual quaternion polynomial
\begin{equation}
  \label{eq:6}
  C = t - r,\quad t \in \R
\end{equation}
parametrizes all rotations with non-vanishing rotation angle around $r$ as well.
In order to also account for $\varphi = 0$, we should extend the parameter range
in \eqref{eq:6} to $\R \cup \{\infty\}$. With the natural understanding that
$C(\infty) \coloneqq \lim_{t \to \infty}\frac{1}{t}C(t) = 1$, the parameter
value $t = \infty$ indeed corresponds to the rotation angle $\varphi = 0$, that is,
the identity transformation.

We would like to emphasize that the dual quaternion in \eqref{eq:4} represents a
rotation with a fixed rotation axis $r$ and rotation angle $\varphi$. The
polynomial in \eqref{eq:6} parametrizes all rotations with fixed rotation axis
$r$. The rotation angle is dependent on the parameter $t$ (we have
$\varphi=2\operatorname{arccot}(-t)$).

More generally, we can consider arbitrary polynomials $C = \sum_{i=0}^d c_it^i$
with coefficients $c_i \in \DH$. Since the indeterminate $t$ typically serves as
a \emph{real} parameter in our context, multiplication, conjugation and
evaluation at real values of polynomials are defined by the conventions that $t$
commutes with all coefficients and $\cj{t} = t$. The thus obtained ring of
polynomials is denoted by $\DH[t]$. Similarly, we can also consider the ring of
bivariate dual quaternion polynomials $\DH[s,t]$ in $s$ and $t$. Its
multiplication, conjugation and evaluation at real values is defined by similar
conventions and the assumption that $s$ and $t$ commute with all coefficients
and with each other.

The linear polynomial $t-r$ from Equation \eqref{eq:6} satisfies $\No{(t-r)} \in
\R[t]$ and also $\No{(t-r)}\neq 0$ (note that $r$ satifies the Plücker
condition). A generalization of this property leads to

\begin{definition}
  \label{def:motion-polynomial}
  A polynomial $C \in \DH[t]$ or in $\DH[s,t]$ is called a \emph{motion
    polynomial} if $C\cj{C} \in \R[t]$ or in $\R[s,t]$, respectively, and
  $C\cj{C} \neq 0$.
\end{definition}

The name ``motion polynomial'' is justified by the observation that the action
\eqref{eq:3} on points allows the parametric version
\begin{equation}
  \label{eq:7}
  x \mapsto y(s,t) = \ej{C}(s,t) x \cj{C}(s,t),\quad s,t \in (\R \cup \{ \infty \}) \times (\R \cup \{ \infty \}).
\end{equation}
Equation~\eqref{eq:7} is a polynomial map in homogeneous coordinates. The
Cartesian coordinates of $y(s,t)$ are rational functions so that \eqref{eq:7}
describes a rigid body motion with rational surfaces as trajectories.

Univariate motion polynomials have been originally defined in \cite{hegedus13}.
There, it was implicitly assumed that motion polynomials are monic. We rather
replace this assumption by the condition $C\cj{C} \neq 0$ which, together with a
proper evaluation at $s = \infty$ or $t = \infty$ and the possibility of
rational re-parametrizations, suffices for our purpose.

\begin{definition}
  \label{def:infinity}
  The value of the motion polynomial $C \in \DH[t]$ at $t = \infty$ is defined
  as $C(\infty) \coloneqq \lim_{t \to \infty}t^{-\deg C}C(t)$. It is the leading
  coefficient of~$C$. The value of $C \in \DH[s,t]$ at $(s,t)$ where $s =
  \infty$ or (not exclusively) $t = \infty$ is defined by similar limits. It is
  the leading coefficient in $s$ or $t$ (or in both), respectively.
\end{definition}

Re-parametrizations that preserve polynomiality and degree of univariate motion
polynomials are maps of the form
\begin{equation}
  \label{eq:8}
  \tau \mapsto \frac{a\tau+b}{c\tau+d},\quad
  a,b,c,d \in \R,\quad
  ad - bc \neq 0,
\end{equation}
combined with multiplying away denominators. With $C = \sum_{i=0}^d c_i t^i$ we have
\begin{equation*}
  C(\tau) = \sum_{i=0}^d c_i (a\tau+b)^i(c\tau+d)^{d-i}.
\end{equation*}
It is noteworthy that \eqref{eq:8} naturally is a map from $\R \cup \{\infty\}$
to $\R \cup \{\infty\}$. Assuming $c \neq 0$, we have
\begin{equation*}
  \infty \mapsto \frac{a}{c}
  \quad\text{and}\quad
  -\frac{d}{c} \mapsto \infty.
\end{equation*}
If $c = 0$, then $\infty$ is a fix point of \eqref{eq:8}. Re-parametrizations of
type \eqref{eq:8} do not change the property of being a motion polynomial.

An extension of \eqref{eq:8} to bivariate polynomials is straightforward. Let us
illustrate some definitions and concepts so far for linear motion polynomials,
which constitute an important special example.

\begin{example}
  \label{ex:linear-motion-polynomial}
  The linear polynomial $C = t - h$ with $h = h_p+\eps h_d$, where $h_p = h_0 + h_1\qi + h_2\qj + h_3\qk$ and $h_d=h_4 + h_5\qi + h_6\qj + h_7\qk$, is a motion polynomial by
  Definition~\ref{def:motion-polynomial} if
  \begin{equation*}
    \No{C} = \No{(t-h)} = t^2 - (h + \cj{h})t + \No{h}
  \end{equation*}
  is a real polynomial. This is equivalent to $h + \cj{h} = 2 (h_0 + \eps h_4)$ and $h\cj{h}
  = h_0^2 + h_1^2 + h_2^2 + h_3^2 + 2\eps(h_0h_4 + h_1h_5 + h_2h_6 + h_3h_7)$
  both being real whence $h_d+\cj{h_d}=0$ and $h_p\cj{h_d}+h_d\cj{h_p}=0$ or, equivalently,
  \begin{equation}
    \label{eq:9}
    h_4 = 0
    \quad\text{and}\quad
    h_1h_5 + h_2h_6 + h_3h_7 = 0.
  \end{equation}
  In this case the motion polynomial $C$ describes a rotation around the
  straight line with Plücker coordinates $\Vect(h) = h_1\qi + h_2\qj + h_3\qk +
  \eps(h_5\qi + h_6\qj + h_7\qk)$, a fact which should not be surprising. We
  already demonstrated the relation between linear motion polynomials and
  rotations. The second equation in \eqref{eq:9} is just the Plücker condition.
  Note that the rotation angle depends on both, $t$ and $h$. By
  Definition~\ref{def:infinity}, the value $C(\infty)$ equals
  $\lim_{t\to\infty}t^{-1}C(t) = 1$ which is the identity displacement.
  Obviously, $C(0) = -h$. The re-parametrization $\tau \mapsto \frac{1}{\tau}$
  is of type \eqref{eq:8} with $a = d = 0$ and $b = c = 1$. It interchanges $0$
  and $\infty$. Indeed, $C(\tau) = 1 - \tau h$ and
  \begin{equation*}
    C(\tau)|_{\tau = 0} = 1,
    \quad
    C(\tau)|_{\tau = \infty} = \lim_{\tau\to\infty} \tau^{-1}C(\tau) = -h,
  \end{equation*}
  as expected.
\end{example}

In the next section we will study bivariate motion polynomials which can be
written as products of linear motion polynomials.

\section{Alternating Factorizations of Bivariate Quaternion Polynomials}
\label{sec:bivariate}

Given a bivariate dual quaternion polynomial $C \in \DH[s,t]$, we denote its
bi-degree by $\bdeg(C)$. We wish to find a motion polynomial $C \in \DH[s,t]$
with $\bdeg(C)=(2,2)$ that admits two different factorizations with alternating
univariate linear factors, i.e.,
\begin{equation}
  \label{eq:10}
  C = (t-h)(s-\ell)(t-m)(s-n)=(s-n')(t-m')(s-\ell')(t-h'),
\end{equation}
where $h$, $\ell$, $m$, $n$, $n'$, $m'$, $\ell'$, $h' \in \DH \setminus \{0\}$
and $\No{(t-r)}=\No{(t-r')}\in \R[t]$, $\No{(s-u)}=\No{(s-u')}\in \R[s]$ for $r
\in \{h, m\}$ and $u \in \{\ell, n\}$, a requirement that is seen to be
necessary by taking norms on both sides of \eqref{eq:10}. We call these
factorizations \emph{alternating} since the $s$- and $t$-factors appear in
alternating order. By the considerations in Section~\ref{sec:preliminaries}, the
linear factors will represent rotations around fixed axes.

Motion polynomials of shape \eqref{eq:10} immediately lead to closed-loop 8R
mechanisms with the properties mentioned in the introduction:
\begin{itemize}
\item Each factorization gives rise to a two-parametric motion of an open 4R
  chain. Since the factorizations agree, the two distal links can be rigidly
  connected to form a closed-loop 8R linkage with the same two degrees of
  freedom.
\item In this two-dimensional motion component (we conjecture that other
  components exist as well), the motion of any axis is determined by either $s$
  or $t$. Locking one axis, that is, fixing $s$ or $t$, automatically locks
  every second axis while the axes parametrized by the other parameter still
  move. In terms of linear motion polynomials, the expression ``locking an
  axis'' can be read as follows: Each linear polynomial in \eqref{eq:10}
  represents a rotation around a fixed axis (c.~f. Example
  \ref{ex:linear-motion-polynomial}). As outlined in the previous section, the
  rotation angle is dependent on the dual quaternion \emph{as well as} on $t$.
  We may now choose a fixed real number $t_0 \in \R$ and consider the
  expressions $t_0-h$, $t_0-m$, $t_0-m'$ and $t_0-h'$. All of them represent
  rotations with a fixed rotation angle and we refer to the respective rotation
  axes as ``locked axes''. By locking all axes parametrized by $t$, we obtain a
  ``sub-mechanism'' with four moving axes (the axes parametrized by $s$) which,
  by a naive counting of degrees of freedom, should be rigid. However, the two
  factorizations in Equation \eqref{eq:10} guarantee that the sub-mechanism
  still moves. Such a closed-loop mobile $4R$ mechanism is called a
  \emph{Bennett linkage} and we use the term ``$s$-Bennett linkage''.
  Interchanging $s$ and $t$ leads to another sub-mechanism, the ``$t$-Bennett
  linkage''.
\end{itemize}

Up to now, only isolated examples of this kind of polynomials have been known
(c.~f. \cite{lercher21,lercher22:_8R-mechanism}). We will present a systematic
construction of these polynomials, and thus of multi-Bennett 8R mechanisms, and
start with a simple yet crucial lemma:

\begin{lem}
  \label{lem:bennett}
  Let $C \in \DH[s,t]$ be a dual quaternion polynomial that admits two
  alternating factorizations:
  \begin{equation}
    \label{eq:11}
    C=(t-h)(s-\ell)(t-m)(s-n)=(s-n')(t-m')(s-\ell')(t-h')
  \end{equation}
  Then
  \begin{equation*}
    (s-\ell)(s-n)=(s-n')(s-\ell') \quad \text{and} \quad (t-h)(t-m)=(t-m')(t-h').
  \end{equation*}
\end{lem}

\begin{proof}
  We may view $C$ as a polynomial in $t$ with coefficients in the ring $\DH[s]$.
  Comparing the coefficient of $t^2$ on the left-hand and the right-hand side of
  Equation~\eqref{eq:11} shows that $(s-\ell)(s-n)=(s-n')(s-\ell')$. The second
  statement follows by interchanging the roles of $s$ and~$t$.
\end{proof}

\begin{rmk}
  \label{rmk:bennett}
  If the polynomial $C = (t-h)(s-\ell)(t-m)(s-n)$ admits a second alternating
  factorization as in \eqref{eq:11}, it can always be computed by so-called
  \emph{Bennett flips} \cite[Definition~4]{li18}. The name is motivated by the
  observation that the revolute axes to $\ell$, $n$, $\ell'$, and $n'$ (and also to $h$, $m$, $h'$ and $m'$) form, in
  that order, a Bennett linkage. More precisely, the quaternions $n'$, $\ell'$,
  $m'$, $h' \in \DH$ can be computed by replacing the univariate polynomials
  $(s-\ell)(s-n)$ and $(t-h)(t-m)$ by their second factorization with linear
  factors. In \cite[Definition~4]{li18}, it is shown that the second
  factorization of a univariate polynomial $(u-h_1)(u-h_2) \in \DH[u]$ is
  obtained via the formulas
  \begin{equation}
    \label{eq:12}
    k_2=-(\cj{h_1}-h_{2})^{-1}(h_1h_2-\No{h_1}) \quad \text{and} \quad k_1=h_1+h_2-k_2,
  \end{equation}
  where $(u-h_1)(u-h_2)=(u-k_1)(u-k_2)$ and $\No{(u-h_1)}=\No{(u-k_2)}$,
  $\No{(u-h_2)}=\No{(u-k_1)}$.
\end{rmk}

In the remainder of this section, we will provide a systematic procedure for the
construction of motion polynomials with two alternating factorizations. We would
like to emphasize that we are not aware of \emph{any} factorization results for
\emph{bivariate} dual quaternion polynomials in existing literature. Our
construction is based on the following idea: In
Section~\ref{subsec:real-quaternions}, we construct bivariate real quaternion
polynomials with two alternating factorizations. In
Section~\ref{subsec:dual-quaternions}, we extend our results to dual quaternion
polynomials.

\subsection{Quaternion polynomials with two alternating factorizations}
\label{subsec:real-quaternions}

The following theorem is the centerpiece of the present section. It presents a
method that can be used to construct quaternion polynomials of bi-degree $(2,2)$
that admit two different factorizations with linear factors.

\begin{thm}
  \label{thm:primalpart}
  Let $h$, $m$, $n \in \H$ be quaternions. Moreover, assume that either
  $\Scal(h) \neq \Scal(m)$ or, if $\Scal(h) = \Scal(m)$, that $\No{h}\neq
  \No{m}$. Then there exists a suitable quaternion $\ell \in \H$ such that the
  polynomial
  \begin{equation}
    \label{eq:13}
    C \coloneqq (t-h)(s-\ell)(t-m)(s-n) \in \H[s,t]
  \end{equation}
  admits a second factorization with univariate linear factors.
\end{thm}

\begin{proof}
  We briefly explain the main idea of the proof: According to \eqref{eq:13}, the
  polynomial $C$ has a left factor of the form $t-h$. By choosing the quaternion
  $\ell$ in a special way, we force the polynomial $C$ to admit another
  factorization with a right factor $t-h'$ of the same norm, that is
  \begin{equation}
    \label{eq:14}
    (t-h)(s-\ell)(t-m)(s-n)=A(t-h') \quad \text{with $A \in \H[s,t]$, $\bdeg(A)=(1,2)$,}
  \end{equation}
  and $\No{(t-h)}=\No{(t-h')}$.
  In \cite{skopenkov19,lercher21}, it is shown that polynomials of degree one in
  $t$ admit factorizations with univariate linear factors as long as the
  corresponding norm polynomial splits into a product of real univariate
  polynomials.\footnote{The original reference is \cite[Lemma~2.9]{skopenkov19},
    but in \cite[p.~9]{lercher21} we provide an algorithm that can be used to
    compute a factorization of the desired form.} This is indeed the case for
  the polynomial $A$ in \eqref{eq:14} since $\No{A}=PR$ with $P=\No{(t-m)}\in
  \R[t]$ and $R=\No{(s-\ell)}\No{(s-n)} \in \R[s]$. Therefore,
  \begin{equation}
    \label{eq:15}
    A=(u_1-h_1)(u_2-h_2)(u_3-h_3) \quad \text{with} \quad u_i \in \{s,t\} \text{ and } h_i \in \H \text{ for } i=1,2,3
  \end{equation}
and $C$ admits a second factorization with univariate linear factors. All possible combinations of linear $s$- and $t$-factors will be discussed in the proof of Corollary~\ref{cor:alt}.

  In order to show \eqref{eq:14}, we define $M \coloneqq \No{(t-h)} \in \R[t]$, view $C$ and $M$ as univariate polynomials with coefficients in $\H[s]$ and apply division with remainder of $C$ by~$M$:
  \begin{equation}
    \label{eq:16}
    C=(t-h)(s-\ell)(t-m)(s-n)=TM+R,
  \end{equation}
  where $T$, $R \in \H[s,t]$ and $\bdeg(R)=(d_t,d_s)$ with $d_t \le 1$ and $d_s
  \le 2$. We compare the coefficients of $t^2$ on the left-hand and right-hand side of
  equation \eqref{eq:16} and conclude $T=(s-\ell)(s-n)$ since $R$ is at most
  linear in $t$. Therefore, the linear factor $s-n$ is a right factor of both
  $C$ and $T$. Representation \eqref{eq:16} then shows that it is also a right
  factor of $R$ (we used the fact $TM=MT$ since the polynomial $M \in \R[t]$ is
  real and commutes with other polynomials). Similarly, the linear factor $t-h$
  is a left factor of both $C$ and $M=(t-h)\cj{(t-h)}$ and hence also a left
  factor of $R$. We can write $R=(t-h)R'=R't-hR' \in \H[s,t]$ with $R' \in
  \H[s]$. Since $s-n$ divides $R$ from the right it needs to divide each
  coefficient of $R$ when viewed as polynomial in $t$ with coefficients in
  $\H[s]$. We conclude that $s-n$ is a right factor of $R'$. Hence $R$ is
  necessarily of the form
  \begin{equation}
    \label{eq:17}
    R=(t-h)(r_1s+r_0)(s-n)
  \end{equation}
  with $r_1, r_0 \in \H$. The quaternions $r_1$ and $r_0$ are obtained by
  comparing appropriate coefficients in \eqref{eq:16}: Up to now, we always considered $C$ and $R$ as univariate polynomials with coefficients in $\H[s]$. We now view them as bivariate polynomials, which allows us to compare the coefficients of $ts^2$ and $t$: Comparing coefficients of
  $ts^2$ in \eqref{eq:16} yields
  \begin{equation*}
    -h-m=-h-\cj{h}+r_1
  \end{equation*}
  and hence $r_1=\cj{h}-m$. Comparing coefficients of $t$ leads to
  \begin{equation*}
    -\ell m n- h\ell n = -\ell nh-\ell n\cj{h}-r_0n
  \end{equation*}
  and hence
  \begin{equation*}
    r_0=[(\ell m + h\ell)n-\ell \underbrace{n(h+\cj{h})}_{\overset{(*)}{=}(h+\cj{h})n}]n^{-1}=\ell m + h\ell -\ell(h+\cj{h})=h\ell-\ell(r_1+h).
  \end{equation*}
  In $(*)$ we used the fact $h+\cj{h} \in \R$. Let us recall the main idea of
  the proof: We need to force $C$ to admit a factorization with the right factor
  $t-h'$, where $h'$ is yet to be determined. Alternatively, we can force $R$ to have the right factor $t-h'$. By
  \eqref{eq:16}, it is then also a right factor of $C$ (note that we require
  $M=\cj{(t-h')}(t-h')$). We write
  \begin{equation*}
    R=r_1(t-r_1^{-1}hr_1)(s+r_1^{-1}r_0)(s-n)=r_1(t-h')(s+r_1^{-1}r_0)(s-n),
  \end{equation*}
  where $h' \coloneqq r_1^{-1}hr_1$. The polynomial $t-h'$ indeed satisfies the required condition $\No{(t-h')}=M$. If the polynomial $S \coloneqq
  (s+r_1^{-1}r_0)(s-n) \in \H[s]$ was a real polynomial, the factor $t-h'$ would
  commute with $S$ and hence be a right factor of $R$. In case $-r_{1}^{-1}r_0 =
  \cj{n}$ we obtain $S = (s-\cj{n})(s-n) \in \R[s]$. Therefore, we need to find
  a quaternion $\ell \in \H$ such that $-r_0=r_1\cj{n}$, that is
  \begin{equation*}
    \ell(r_1+h)-h\ell = r_1\cj{n}.
  \end{equation*}
  Above equation is a \emph{linear} equation in the quaternion unknown $\ell \in
  \H$. By \cite[Theorem 2.3]{janovska08} it is uniquely solvable if and only if
  $\Scal(A) \neq -\Scal(B)$ or $\Scal(A) = -\Scal(B)$ and $\No{A} \neq \No{B}$,
  where $A \coloneqq -h$ and $B \coloneqq h+r_1$. This is equivalent to our
  theorem's assumption $\Scal(h) \neq \Scal(m)$ or $\Scal(h) = \Scal(m)$ and
  $\No{h}\neq \No{m}$. In the referenced Theorem~2.3 of \cite{janovska08}, an
  explicit formula for the solution $\ell \in \H$ is provided:
  \begin{equation*}
    \ell = (2(2\Scal(h)-\Scal(m))-h-(h+r_1)\cj{(h+r_1)}h^{-1})^{-1}(r_1\cj{n}-h^{-1}r_1\cj{n}\cj{(h+r_1)}).
  \end{equation*}
  This proves the claim.
\end{proof}

In order to construct mechanisms, we need to guarantee that the second
factorization in Theorem \ref{thm:primalpart} is alternating as well. This is
ensured by some additional assumptions stated in the (rather technical)
Corollary \ref{cor:alt}.

\begin{cor}
  \label{cor:alt}
  Let $h$, $m$, $n \in \H$ be quaternions and $h'$, $m' \in \H$ be such that
  $(t-h)(t-m)=(t-m')(t-h')$ (c.~f. Remark \ref{rmk:bennett}). If $mn\neq nm$ and
  $h'n\neq nh'$, the second factorization in Theorem \ref{thm:primalpart} is
  alternating as well, that is
  \begin{equation*}
    C = (t-h)(s-\ell)(t-m)(s-n) = (s-n')(t-m')(s-\ell')(t-h'),
  \end{equation*}
  where $\No{(t-r)}=\No{(t-r')}$, $\No{(s-u)}=\No{(s-u')}$ for $r \in \{h, m\}$
  and $u \in \{\ell, n\}$.
\end{cor}

\begin{proof}
  We have $3!=6$ possibilities for factorizations of $A$ with univariate linear
  factors, where $A$ is defined in \eqref{eq:15}. We use representation
  \eqref{eq:14} and obtain
  \begin{align}
    C &= (t-h)(s-\ell)(t-m)(s-n)=\mathbf{(s-\bll)(t-m')(s-n)}(t-h'), \label{it:1}\\
    C &= (t-h)(s-\ell)(t-m)(s-n)=\mathbf{(t-m')(s-\bll)(s-n)}(t-h'), \label{it:2}\\
    C &= (t-h)(s-\ell)(t-m)(s-n)=\mathbf{(s-\bll)(s-n)(t-m')}(t-h'), \label{it:3}\\
    C &= (t-h)(s-\ell)(t-m)(s-n)=\mathbf{(t-m')(s-n')(s-\bll')}(t-h'), \label{it:4}\\
    C &= (t-h)(s-\ell)(t-m)(s-n)=\mathbf{(s-n')(s-\bll')(t-m')}(t-h'),\quad\text{or} \label{it:5}\\
    C &= (t-h)(s-\ell)(t-m)(s-n)=\mathbf{(s-n')(t-m')(s-\bll')}(t-h'). \label{it:6}
  \end{align}
  We highlighted the different possibilities for factorizations of $A$ by using
  bold letters. Note that the linear $s$-factors on the left have the same norms
  as the linear $s$-factors on the right but the $s$-factors possibly appear in a different order. If
  the order is different, the $s$-factors must correspond in Bennett flips
  by arguments as in the proof of Lemma~\ref{lem:bennett} and are denoted by a
  prime, i.e. $n'$, $\ell'$. If the order is the same, the
  $s$-factors are equal by arguments similar to Lemma~\ref{lem:bennett} and
  \cite[Lemma~3]{hegedus13}. The same arguments apply to linear $t$-factors.

  The two factorizations in \eqref{it:1} are $t$-equivalent in the sense of \cite[Definition~4.3]{lercher21}.\footnote{In \cite[Definition~4.3]{lercher21}, two different factorizations of
    bivariate quaternion polynomials with linear factors are called
    $t$-equivalent, if the linear $s$-factors appear in the same order. This is
    the case in \eqref{it:1} and \eqref{it:2} since $s-\ell$ is the first and
    $s-n$ the second $s$-factor in both factorizations. Such factorizations are
    special since they can be transferred into each other by applying Bennett
    flips and letting appropriate $s$- and $t$-factors commute with each other
    (c.~f. \cite[Proposition~4.6]{lercher21}).} By \cite[Proposition~4.6]{lercher21}, we conclude $h'n=nh'$ (and also
  $h\ell=\ell h$), a case which is excluded by assumption. The same can be said for the two factorizations in~\eqref{it:2}. The second
  factorization in \eqref{it:3} can be rewritten as $(s-\ell)(s-n)(t-h)(t-m)$
  and therefore turns out to be $s$-equivalent to the first factorization in
  \eqref{it:3}. We again use \cite[Proposition~4.6]{lercher21} and conclude
  $mn=nm$, which is also excluded by assumption. The second factorizations in
  \eqref{it:4} and also in \eqref{it:5} are coincident with the second factorizations in
  \eqref{it:2} and \eqref{it:3} after applying Bennett flips of
  $(s-n')(s-\ell')$. Therefore, $C$ needs to admit two different factorizations
  of the form~\eqref{it:6}.
\end{proof}

\begin{rmk}
  Under the weak assumptions of Corollary~\ref{cor:alt},
  Theorem~\ref{thm:primalpart} guarantees existence of a quaternion $\ell \in
  \H$ such that $C=(t-h)(s-\ell)(t-m)(s-n)$ admits a second alternating
  factorization. While our proofs are constructive, the actual computation of
  the second factorization can be simplified a lot with the help of
  Remark~\ref{rmk:bennett}. At first, we compute quaternions $h'$, $\ell'$,
  $m'$, $n' \in \H$ via Bennett flips \eqref{eq:12} of the univariate
  polynomials $(t-h)(t-m)$ and $(s-\ell)(s-n)$, respectively. The second
  factorization is then given by $C=(s-n')(t-m')(s-\ell')(t-h')$. Pseudocode for
  this approach is given in Algorithm~\ref{alg:2alt}.
\end{rmk}

\begin{algorithm}
\caption{Polynomials with two alternating factorizations}
\label{alg:2alt}
\begin{algorithmic}[1]
\Require Quaternions $h$, $m$, $n \in \H$ satisfying the assumptions of Theorem~\ref{thm:primalpart} and Corollary~\ref{cor:alt}.
\Ensure Two tuples $(t-h,s-\ell,t-m,s-n)$ and $(s-n',t-m',s-\ell',t-h')$ such that $(t-h)(s-\ell)(t-m)(s-n)=(s-n')(t-m')(s-\ell')(t-h')$.
\State $r_1 \gets \cj{h}-m$
\State $\ell \gets (2(2\Scal(h)-\Scal(m))-h-(h+r_1)\cj{(h+r_1)}h^{-1})^{-1}(r_1\cj{n}-h^{-1}r_1\cj{n}\cj{(h+r_1)})$
\State $h' \gets -(\cj{h}-m)^{-1}(hm-\No{h}), \ m' \gets h+m-h'$
\State $\ell' \gets -(\cj{\ell}-n)^{-1}(\ell n-\No{\ell}), \ n' \gets \ell+n-\ell'$\\
\Return $(t-h,s-\ell,t-m,s-n), \ (s-n',t-m',s-\ell',t-h')$
\end{algorithmic}
\end{algorithm}

\begin{example}
  \label{ex:primal}
  Setting
  \begin{equation*}
      h = 2\qi-\qj-3\qk,\quad
      m = -6-2\qi+3\qj-3\qk,\quad
      n = -\qj
  \end{equation*}
  and applying Algorithm~\ref{alg:2alt} yields the two alternating
  factorizations
  \begin{multline*}
    (t - 2\qi + \qj + 3\qk)(s + \qi - \qj)(t + 6 + 2\qi - 3\qj + 3\qk)(s + \qj)\\
    =(s - \qj)(t - 2\qi - 3\qj + 3\qk + 6)(s + \qi + \qj)(t + 2\qi + \qj + 3\qk).
  \end{multline*}
\end{example}

\subsection{An extension to dual quaternion polynomials}
\label{subsec:dual-quaternions}

When it comes to applications in space kinematics, it is necessary to formulate
our statements for \emph{dual quaternion polynomials}. The extension of two
different alternating factorizations over the quaternions to dual quaternions is
straightforward by using the following approach:

\begin{description}
\item[Step 1:] We start with two dual quaternions $h = h_p + \eps h_d$ and $m =
  m_p + \eps m_d$ such that $t - h$ and $t - m$ satisfy the motion polynomial
  condition of Definition~\ref{def:motion-polynomial} and we compute Bennett
  flips of $h$ and $m$ to obtain dual quaternions $h'=h_p'+\varepsilon h_d'$ and
  $m'=m_p'+\varepsilon m_d'$, respectively.
\item[Step 2:] We choose a quaternion $n_p \in \H$ such that $h_p$, $m_p$,
  and $n_p$ satisfy the conditions of Corollary~\ref{cor:alt} and apply
  Theorem~\ref{thm:primalpart} to the quaternions $h_p$, $m_p$, and $n_p$. This
  gives a quaternion polynomial $C_p \in \H[s,t]$ that admits the two different
  alternating factorizations
  \begin{equation}
    \label{eq:18}
    C_p=(t-h_p)(s-\ell_p)(t-m_p)(s-n_p) = (s-n_p')(t-m_p')(s-\ell_p')(t-h_p').
  \end{equation}
\item[Step 3:] Finally, we have to determine the respective dual parts $\ell_d$,
  $n_d$, $\ell_d'$, $n_d' \in \H$ of the dual quaternions $\ell = \ell_p + \eps
  \ell_d$, $n = n_p + \eps n_d$, $\ell' = \ell'_p + \eps \ell'_d$, and $n' = n'_p +
  \eps n'_d$, to allow for two factorizations of
  \begin{multline}
    \label{eq:19}
    C\coloneqq(t-h_p-\varepsilon h_d)(s-\ell_p-\varepsilon \mathbf{\bll_d})(t-m_p-\varepsilon m_d)(s-n_p-\varepsilon \mathbf{n_d}) =\\
    (s-n_p'-\varepsilon \mathbf{n_d'})(t-m_p'-\varepsilon m_d')(s-\ell_p'-\varepsilon \mathbf{\bll_d'})(t-h_p'-\varepsilon h_d').
  \end{multline}
  The yet unknown quaternions are highlighted in bold letters. Comparing
  coefficients in $t$ and $s$ for all quaternion coefficients on the left-hand
  and right-hand side of equation \eqref{eq:19} yields a system of $32$
  equations in $16$ unknowns. (Note that the primal parts are equal by
  construction.) Additionally, we have to impose the motion polynomial
  conditions of Definition~\ref{def:motion-polynomial} on the linear
  $s$-polynomials, leading to eight further linear equations in $16$ unknowns (c.~f. Example~\ref{ex:linear-motion-polynomial}):
  \begin{equation*}
    \begin{split}
      \ell_p\mathbf{\cj{\bll_d}}+\mathbf{\bll_d}\cj{\ell_p}=0, \quad \mathbf{\bll_d}+\cj{\mathbf{\bll_d}}=0,\\
      \ell_p'\cj{\mathbf{\bll_d'}}+\mathbf{\bll_d'}\cj{\ell_p'}=0, \quad \mathbf{\bll_d'}+\cj{\mathbf{\bll_d'}}=0,\\
      n_p\cj{\mathbf{n_d}}+\mathbf{n_d}\cj{n_p}=0, \quad \mathbf{n_d}+\cj{\mathbf{n_d}}=0,\\
      n_p'\cj{\mathbf{n_d'}}+\mathbf{n_d'}\cj{n_p'}=0, \quad \mathbf{n_d'}+\cj{\mathbf{n_d'}}=0.
    \end{split}
  \end{equation*}
  In total, we have to solve a system of $40$ linear equations in $16$ unknowns.
  The respective linear system of equations seems to be highly overconstrained.
  Quite surprisingly, it turns out to always admit a solution. This will be
  proved by a straightforward computation in Section~\ref{sec:DH-parameters} so that we have:
\end{description}

\begin{thm}
  \label{th:dual-extension}
  The construction outlined in above Steps~1 to 3 generically yields a motion polynomial $C$ satisfying
  \begin{equation*}
    C = (t-h)(s-\ell)(t-m)(s-n) = (s-n')(t-m')(s-\ell')(t-h')
  \end{equation*}
  with linear motion polynomials $t-h$, $s-\ell$, \ldots, $t-h'$.
\end{thm}

\begin{example}
  We build on Example \ref{ex:primal} and additionally choose quaternions
  \begin{equation*}
    h_d \coloneqq 23\qi - 74\qj + 40\qk \quad \text{and} \quad m_d \coloneqq -45\qi - 66 \qj - 36\qk.
  \end{equation*}
  The polynomials $t-h_p-\varepsilon h_d$ and $t-m_p-\varepsilon m_d$ are motion
  polynomials:
  \begin{equation*}
    \begin{aligned}
      \No{(t-h_p-\varepsilon h_d)} &= t^2+14 \in \R[t] \setminus \{0\}, \\
      \No{(t-m_p-\varepsilon m_d)} &= t^2+12t+58 \in \R[t] \setminus \{0\}.
    \end{aligned}
  \end{equation*}
  We compute Bennett flips of $h \coloneqq h_p+\varepsilon h_d$ and $m \coloneqq
  m_p+\varepsilon m_d$ and obtain
  \begin{equation*}
    \begin{aligned}
      m' =& -6 + 2\qi + 3\qj - 3\qk - \varepsilon (21\qi + 22\qj + 36\qk)\\
      h' =& -2\qi - \qj - 3\qk - \varepsilon (\qi + 118\qj - 40\qk).
    \end{aligned}
  \end{equation*}
  The unknowns $\ell_d$, $n_d$, $\ell_d'$, $n_d'$ are obtained by solving the
  respective system of linear equations:
  \begin{alignat*}{2}
    \ell_d &= -11\qi - 11\qj + 2\qk, &\quad n_d &= -3\qi - 2\qk,\\
    \ell_d' &= 11\qi - 11\qj - 2\qk, &\quad n_d' &= -25\qi + 2\qk.
  \end{alignat*}
  Finally, we get a motion polynomial in $\DH[s,t]$ with two alternating
  factorizations:
  \begin{multline*}
    (t - 2\qi + \qj + 3\qk + \varepsilon (-23\qi + 74\qj - 40\qk))(s + \qi - \qj + \varepsilon (11\qi + 11\qj - 2\qk))\\
    (t + 6 + 2\qi - 3\qj + 3\qk + \varepsilon(45\qi + 66\qj + 36\qk))(s + \qj + \varepsilon(3\qi + 2\qk)) \\
    = (s - \qj + \varepsilon(25\qi - 2\qk))(t+6 - 2\qi - 3\qj + 3\qk +  \varepsilon(21\qi + 22\qj + 36\qk))\\
    (s+\qi+\qj-\varepsilon(11\qi - 11\qj - 2\qk))(t+2\qi + \qj + 3\qk + \varepsilon(\qi + 118\qj - 40\qk)).
  \end{multline*}
\end{example}

\section{The Multi-Bennett 8R Mechanism}
\label{sec:8R-Mechanism}

In the preceding section we proved existence of bivariate quaternion polynomials
$C \in \H[s,t]$ that admit two factorizations with linear quaternion polynomials and we hinted at the possibility to extend this to motion
polynomials of the shape
\begin{equation*}
  C = (t-h)(s-\ell)(t-m)(s-n) = (s-n')(t-m')(s-\ell')(t-h')
\end{equation*}
with \emph{dual quaternions} $h$, $\ell$, $m$, $n$, $n'$, $m'$, $\ell'$, and
$h'$. By construction, each linear factor in $t$ or in $s$ parametrizes a
rotation around a straight line in space so that each of the two factorizations
gives rise to an open 4R chain whose end-effectors share the two-parametric
rational motion parametrized by $C$. Thus, this motion is contained in the
configuration variety of the closed-loop 8R linkage formed by the two open 4R
chains. Investigation of properties of this 8R linkage is the topic of this
section. In doing so, we only consider the \emph{generic} case, i.e., we assume
that no special algebraic relations between the input parameters are fulfilled.
At present, a comprehensive discussion of all special cases seems of little
value.

The linkage's zero configuration is given by $t = s = \infty$ because then there
is zero rotation in all joints (c.f. Definition~\ref{def:infinity}). The axes'
Plücker coordinates in this zero configuration are simply the respective
vector parts
\begin{gather*}
  \Vect(h),\quad
  \Vect(\ell),\quad
  \Vect(m),\quad
  \Vect(n),\quad
  \Vect(n'),\quad
  \Vect(m'),\quad
  \Vect(\ell'),\quad
  \Vect(h')
\end{gather*}
of the linear factors. The positions of these axes in the
configuration determined by a general parameter pair $(s,t)$ can be computed via
\eqref{eq:5} as
\begin{equation}
  \label{eq:20}
  \begin{aligned}
    H(s,t) &= \Vect(h),\\
    L(s,t) &= \ej{\bigl((t-h)\Vect(\ell)(t-\cj{h})\bigr)},\\
    M(s,t) &= \ej{\bigl((t-h)(s-\ell)\Vect(m)(s-\cj{\ell})(t-\cj{h})\bigr)},\\
    N(s,t) &= \ej{\bigl((t-h)(s-\ell)(t-m)\Vect(n)(t-\cj{m})(s-\cj{\ell})(t-\cj{h})\bigr)},\\
    N'(s,t) &= \Vect(n'),\\
    M'(s,t) &= \ej{\bigl((s-n')\Vect(m')(s-\cj{n'})\bigr)},\\
    L'(s,t) &= \ej{\bigl((s-n')(t-m')\Vect(\ell')(t-\cj{m'})(s-\cj{n'})\bigr)},\\
    H'(s,t) &= \ej{\bigl((s-n')(t-m')(s-\ell')\Vect(h')(s-\cj{\ell'})(t-\cj{m'})(s-\cj{n'})\bigr)}.
\end{aligned}
\end{equation}
Note that $H(s,t)$ and $N'(s,t)$ are independent of $t$ and $s$, $L(s,t)$
depends only on $t$ and $M'(s,t)$ depends only on~$s$.

For fixed $t = t_0$, the axes $L(s,t_0)$, $N(s,t_0)$, $L'(s,t_0)$, and
$N'(s,t_0)$ form, in that order, a Bennett linkage
whose motion is parametrized by $s$. We call it the \emph{$s$-Bennett linkage
  at $t_0$.} Similarly, for fixed $s = s_0$ we obtain a \emph{$t$-Bennett
  linkage at $s_0$,} formed by $H(s_0,t)$, $M(s_0,t)$, $H'(s_0,t)$, and
$M'(s_0,t)$.

It is well-known (and follows from Bennett's original description of his
mechanism as \emph{isogram}, c.f. \cite[Section~10.3]{odehnal20}) that for given
$t_0$ there exist two values for $s \in \R \cup \{\infty\}$ at which the
$t_0$-Bennett mechanism is in a configuration where its four axes have the same
common perpendicular. We call this an \emph{aligned configuration.} A similar
statement holds true for every $s_0$-Bennett mechanism.

The aligned configurations will play a crucial role in our computation of the
8R-linkage's Denavit-Hartenberg parameters in the next section. The 8R-linkage
itself exhibits an interesting aligning behavior as well that will be
investigated in more detail in the forthcoming
Section~\ref{sec:Bennett-sub-mechanisms}.

\subsection{Denavit-Hartenberg Parameters}
\label{sec:DH-parameters}

The aim of this section is the proof of simple relations among the
multi-Bennett's Denavit-Hartenberg parameters. In order to do so, we will
compute parametrizations of its moving axes with respect to special coordinates.
None of these assumptions is a loss of generality so that the resulting
statements are of general validity and are suitable for proving the missing
piece in Theorem~\ref{th:dual-extension}.

Our computation of the 8R-linkage's Denavit-Hartenberg parameters will profit a
lot from the geometry of its $t$- and $s$-Bennett linkages. According to
\cite[Section~10.3]{odehnal20}, the axes of any Bennett linkage can be computed
by
\begin{itemize}
\item picking two arbitrary points $F_h$, $F_m$ and a straight line $z$,
\item rotating $F_h$ and $F_m$, respectively, around $z$ by a rotation angle of
  $\ang{180}$ to obtain a spatial quadrilateral $F_h$, $F_m$, $F_{h'}$,
  $F_{m'}$ with equal opposite sides, and
\item selecting the axes $H$, $M$, $H'$, $M'$ as the perpendiculars to the
  quadrilateral's sides at $F_h$, $F_m$, $F_{m'}$, and $F_{h'}$, respectively
  (Figure~\ref{fig:bennett-linkage}).
\end{itemize}

\begin{figure}
  \centering
  \includegraphics{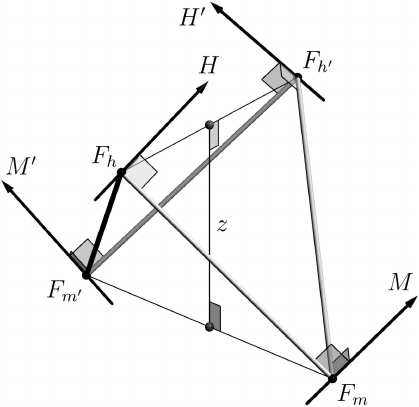}
\caption{Axis configuration of a Bennett linkage}
  \label{fig:bennett-linkage}
\end{figure}

In order to compute the linkage's Denavit-Hartenberg parameters, we assume that
the $t$-Bennett linkage at $s_0 = \infty$ is aligned for $t_0 = \infty$. This is
no loss of generality as it can be achieved via re-parametrizations of type
\eqref{eq:8}. Moreover, we assume that the axes in this configuration intersect
the first coordinate axis perpendicularly and that the axis of half-turn
symmetry is the third coordinate axis. This entails a slight alteration of the
construction from above. We assign coordinates
\begin{equation*}
  F_h = (h_1, 0, 0),\quad
  F_m = (m_1, 0, 0),\quad
  F_{h'} = (-h_1, 0, 0),\quad
  F_{m'} = (-m_1, 0, 0)
\end{equation*}
to the common normal feet and
\begin{equation*}
  D_h = (0, h_2, h_3 ),\quad
  D_m = (0, m_2, m_3),\quad
  D_{h'} = (0, -h_2, h_3),\quad
  D_{m'} = (0, -m_2, m_3).
\end{equation*}
to the corresponding axis directions. By this choice, we ensure equal opposite
distance and angles but not equality of Bennett ratios.\footnote{An important characteristic of a Bennett mechanism is its \emph{Bennett ratio,}
the ratio between sine of angle and distance of two consecutive axes, which is
independent of the chosen pair of consecutive axes
\cite[Equation~(11.69)]{mccarthy11}.} A straightforward
computation yields that this can be satisfied by
\begin{equation*}
  m_1 = \frac{h_1h_3m_2}{h_2m_3}
  \quad\text{or}\quad
  m_1 = \frac{h_1h_2m_3}{h_3m_2}.
\end{equation*}
Since both expressions are equal up to interchanging $h_2$ with $h_3$ and $m_2$
with $m_3$ we can use either of them. The following computations use $m_1 =
h_1h_3m_2/(h_2m_3)$.

Now, we compute the Plücker coordinates, viewed as dual quaternions, of the axes
in the zero configuration as
\begin{equation}
  \label{eq:21}
  \begin{aligned}
    \Vect(h) &= D_h + \eps (F_h \times D_h) = h_2\qj + h_3\qk - h_1\eps(h_3\qj - h_2\qk), \\
    \Vect(m) &= D_m + \eps (F_m \times D_m) = m_2\qj + m_3\qk - \frac{h_1h_3m_2}{h_2m_3}\eps(m_3\qj - m_2\qk), \\
    \Vect(h') &= D_{h'} + \eps (F_{h'} \times D_{h'}) = -h_2\qj + h_3\qk + h_1\eps(h_3\qj + h_2\qk), \\
    \Vect(m') &= D_{m'} + \eps (F_{m'} \times D_{m'}) = -m_2\qj + m_3\qk + \frac{h_1h_3m_2}{h_2m_3}\eps(m_3\qj + m_2\qk).
  \end{aligned}
\end{equation}
Here, we identified in the usual way vectors in $\R^3$ with vectorial
quaternions. The coefficients $h$, $m$, $h'$, and $m'$ in the factors $t-h$,
$t-m$, $t-h'$, $t-m'$ of the sought motion polynomial are linear combinations of
$1$ and $\Vect(h)$, $\Vect(m)$, $\Vect(h')$, and $\Vect(m')$, respectively. The
coefficients cannot be chosen arbitrarily but are subject to the closure
condition $(t-h)(t-m) = (t-m')(t-h')$. This is ensured by having
\begin{equation*}
  \begin{aligned}
    h &= \mu - \nu m_2 (\qj + \frac{h_3}{h_2}\qk) + \frac{\nu h_1m_2}{h_2}\eps (h_3\qj - h_2\qk),\\
    m &= \mu + \nu (m_2\qj + m_3\qk) - \frac{\nu h_1h_3m_2}{h_2m_3}\eps (m_3\qj - m_2\qk),\\
    h' &= \mu + \nu m_2 (\qj - \frac{h_3}{h_2}\qk) - \frac{\nu h_1m_2}{h_2}\eps (h_3\qj + h_2\qk),\\
    m' &= \mu - \nu (m_2\qj - m_3\qk) + \frac{\nu h_1h_3m_2}{h_2m_3}\eps (m_3\qj + m_2\qk)
  \end{aligned}
\end{equation*}
with parameters $\nu$, $\mu \in \R$.

So far, we have followed Step~1 of Section~\ref{sec:bivariate} and computed, in
full generality but at a special configuration, the axes and corresponding dual
quaternions that move with parameter $t$. For Steps~2 and 3 we make the general
\emph{ansatz}
\begin{equation*}
  \ell = \ell_p + \eps \ell_d,\quad
  n = n_p + \eps n_d,\quad
  \ell' = \ell'_p + \eps \ell'_d,\quad
  n' = n'_p + \eps n'_d
\end{equation*}
with $\ell_p$, $\ell_d$, $n_p$, $n_d$, $\ell'_p$, $\ell'_d$, $n'_p$, $n'_d \in
\H$. Step~2 gives the primal parts $\ell_p$, $\ell'_p$, and $n_p'$ in terms of the
indetermined coefficients of $n_p = n_0 + n_1\qi + n_2\qj + n_3\qk$:
\begin{equation}
  \label{eq:22}
  \begin{aligned}
    \ell_p &=
             \frac{1}{h_2 m_3 + h_3 m_2} \Bigl(
             h_2(m_3n_0-2m_2n_1)+h_3m_2n_0
             + n_1(h_2m_3-h_3m_2) \qi \\
           &\qquad\qquad\hfill
             + (h_2(m_3n_2-2m_2n_3)-h_3m_2n_2) \qj
             - n_3(h_2m_3+h_3m_2) \qk
             \Bigr),
             \\
    \ell'_p &=
              \frac{1}{h_2m_3+h_3m_2} \Bigl(
              h_2(m_3n_0-2m_2n_1)+h_3m_2n_0
              + n_1(h_2m_3-h_3m_2) \qi  \\
             &\qquad\qquad\hfill
              + (h_2(m_3n_2-2m_2n_3)-h_3m_2n_2) \qj
              + n_3(h_2m_3+h_3m_2) \qk
              \Bigr),\\
    n_p' &= n_0 + n_1\qi + n_2\qj - n_3\qk.
  \end{aligned}
\end{equation}
This ensures that the primal parts on both sides of
\begin{equation*}
  (t - h)(s - \ell)(t - m)(s - n) =
  (s - n')(t - m')(s - \ell')(t - h')
\end{equation*}
agree. Equality of the respective dual parts together with the motion polynomial
condition boils down to a system of linear equations (Step~3) for the real
coefficients of $\ell_d$, $\ell'_d$, $n_d$, and $n'_d$ which we solve with a
computer algebra system. There is, indeed, a unique solution whence we have
provided the missing piece in the proof of Theorem~\ref{th:dual-extension}.

The solutions are just a bit too long to be displayed here. Therefore, and also
having in mind forthcoming computations, we strive for further simplifications.
By a rational re-parametrization we can achieve that the revolute axis
$\Vect(n)$ is perpendicular to the first coordinate axis in the zero
configuration, at $s_0 = \infty$. This having done, we see that necessarily $n_1
= 0$. With this admissible simplification, the solutions for the dual parts are
\begin{multline*}
  \ell_d = \frac{h_1n_2}{\Delta}
  \bigl(
  n_3(h_2m_3+h_3m_2)(h_2m_2n_3-h_2m_3n_2+h_3m_2n_2+h_3m_3n_3) \qj\\
  -(h_2m_2n_3-h_2m_3n_2+h_3m_2n_2+h_3m_3n_3)(2h_2m_2n_3-h_2m_3n_2+h_3m_2n_2)\qk
  \bigr),
\end{multline*}
\begin{multline*}
  n_d = \frac{h_1}{\Delta}
  \bigl(
  -n_3(2h_2m_2n_3-h_2m_3n_2+h_3m_2n_2)(h_2m_2n_3-h_2m_3n_2+h_3m_2n_2+h_3m_3n_3) \qj \\
  + n_2(2h_2m_2n_3-h_2m_3n_2+h_3m_2n_2)(h_2m_2n_3-h_2m_3n_2+h_3m_2n_2+h_3m_3n_3) \qk
  \bigr),
\end{multline*}
\begin{multline*}
  \ell'_d = \frac{h_1n_2}{\Delta}
  \bigl(
  n_3(h_2m_3+h_3m_2)(h_2m_2n_3-h_2m_3n_2+h_3m_2n_2+h_3m_3n_3)\qj \\
  + (h_2m_2n_3-h_2m_3n_2+h_3m_2n_2+h_3m_3n_3)(2h_2m_2n_3-h_2m_3n_2+h_3m_2n_2) \qk
  \bigr)
\end{multline*}
\begin{multline*}
  n'_d = \frac{h_1}{\Delta}
  \bigl(
  -n_3(2h_2m_2n_3-h_2m_3n_2+h_3m_2n_2)(h_2m_2n_3-h_2m_3n_2+h_3m_2n_2+h_3m_3n_3) \qj \\
  -n_2(2h_2m_2n_3-h_2m_3n_2+h_3m_2n_2)(h_2m_2n_3-h_2m_3n_2+h_3m_2n_2+h_3m_3n_3) \qk
  \bigr)
\end{multline*}
where
\begin{equation*}
  \Delta = h_2m_3((h_3m_2-h_2m_3)(n_2^2-n_3^2)+2(h_2m_2+h_3m_3)n_2n_3).
\end{equation*}

But having $n_1 = 0$ has further consequences:
\begin{itemize}
\item A glance at \eqref{eq:22} immediately confirms that all coefficients of
  $\qi$ vanish for all revolute axes in the zero configuration. Therefore, all
  revolute axes in the zero configuration \emph{are perpendicular to the first
    coordinate axis.}
\item It can readily be verified that the intersection conditions
  \begin{multline*}
    \Vect(\ell) \qi - \qi\cj{\Vect(\ell)} =
    \Vect(n) \qi - \qi\cj{\Vect(n)} \\
    = \Vect(\ell') \qi - \qi\cj{\Vect(\ell')} =
    \Vect(n') \qi - \qi\cj{\Vect(n')} = 0.
  \end{multline*}
  between the first coordinate axis (with Plücker coordinates $\qi$) and all
  mechanism axes that move with parameter $s$ in the zero configuration are
  satisfied.
\end{itemize}

This means that in the zero configuration \emph{all revolute axes intersect the
  first coordinate axis perpendicularly.} We infer that not only the $t$-Bennett
mechanism but also the $s$-Bennett mechanism aligns and both share the common
perpendicular of their axis. Since each Bennett mechanism has two aligned
configurations and there is nothing special about our zero configuration, we can
say:

\begin{thm}
  \label{th:folded}
  The multi-Bennett 8R mechanism has four aligned configurations in which all
  eight revolute axes share a common perpendicular line.
\end{thm}

The four aligned configurations of an example can be seen in the corners of
Figure~\ref{fig:8R}.

From the representations \eqref{eq:21} and \eqref{eq:22} of the axes' Plücker
coordinates, it is straightforward to compute the mechanism's Denavit-Hartenberg
parameters. The information given in \cite[Section~2.1.2]{pottmann10} is
sufficient for that purpose but more explicit formulas are also available, for
example in \cite{faria19}. Using computer algebra, it is easy to verify

\begin{thm}
  \label{th:DH-relations}
  The offsets of a multi-Bennett 8R mechanism are all zero. Opposite distances as
  well as opposite angles are equal.
\end{thm}

Remarkably, the four distances are rational expressions in the input
parameters, no square roots appear:
\begin{equation*}
  \begin{aligned}
    d_1 &= \frac{1}{\Phi}
    (h_1(m_2n_3-m_3n_2)(2h_2^2m_2n_3-h_2^2m_3n_2+h_2h_3m_2n_2+h_2h_3m_3n_3+h_3^2m_2n_3)),\\
    d_2 &= \frac{1}{\Phi}
    (h_1(m_2n_3-m_3n_2)(h_2n_2-h_3n_3)(h_2m_3-h_3m_2)),\\
    d_3 &= \frac{1}{\Phi}
    (h_1(m_2n_2+m_3n_3)(h_2n_3+h_3n_2)(h_2m_3-h_3m_2)),\\
    d_4 &= \frac{1}{\Phi}
    (h_1(h_2n_3+h_3n_2)(2h_2m_2^2n_3-h_2m_2m_3n_2+h_2m_3^2n_3+h_3m_2^2n_2+h_3m_2m_3n_3)),
  \end{aligned}
\end{equation*}
where
\begin{equation*}
  \Phi = h_2m_3((h_3m_2-h_2m_3)(n_2^2-n_3^2) + 2n_2n_3(h_2m_2+h_3m_3)).
\end{equation*}

The squared cosines of the corresponding angles are
\begin{equation*}
  \begin{aligned}
    \cos^2\alpha_1 &= \frac{(m_2n_2+m_3n_3)^2}{(n_2^2+n_3^2)(m_2^2+m_3^2)},\\
    \cos^2\alpha_2 &= \frac{(2h_2m_2^2n_3-h_2m_2m_3n_2+h_2m_3^2n_3+h_3m_2^2n_2+h_3m_2m_3n_3)^2}
                           {(m_2^2+m_3^2)\Psi}\\
    \cos^2\alpha_3 &= \frac{(2h_2^2m_2n_3-h_2^2m_3n_2+h_2h_3m_2n_2+h_2h_3m_3n_3+h_3^2m_2n_3)^2}
                           {(h_2^2+h_3^2)\Psi},\\
    \cos^2\alpha_4 &= \frac{(h_2n_2-h_3n_3)^2}{(h_2^2+h_3^2)(n_2^2+n_3^2)}
  \end{aligned}
\end{equation*}
where
\begin{multline*}
  \Psi = 4h_2^2m_2n_3(m_2n_3-m_3n_2) + (h_2^2m_3^2 + h_3^2m_2^2)(n_2^2+n_3^2)\\
  + 2h_2h_3m_2(2m_2n_2n_3-m_3(n_2^2-n_3^2)).
\end{multline*}

We conjecture that the necessary conditions of Theorem~\ref{th:DH-relations} on
the mechanism's Denavit-Hartenberg parameters are not sufficient to characterize
a multi-Bennett 8R mechanism.

\subsection{Bennett Sub-Mechanisms}
\label{sec:Bennett-sub-mechanisms}

We have already mentioned that for fixed $s = s_0$ the axes $H(s,t)$, $M(s,t)$,
$H'(s,t)$, and $M'(s,t)$ to the respective factors $t-h$, $t-m$, $t-h'$, and
$t-m'$ form a Bennett mechanism. The same is true for fixed $t = t_0$ and the
axes $L(s,t)$, $N(s,t)$, $L'(s,t)$, $N'(s,t)$ to the respective factors
$s-\ell$, $s-n$, $s-\ell'$, $s-n'$. We refer to the respective Bennett
mechanisms as $t$-Bennett mechanism at $s_0$ and as $s$-Bennett mechanism at
$t_0$. The $t$-Bennett mechanism aligns for precisely two parameter values $t'$,
$t''$. By means of \eqref{eq:20} it can readily be verified that aligning of
$t$-Bennett linkage happens at
\begin{equation}
  \label{eq:23}
  t' = \infty, \quad t'' = \mu
\end{equation}
while an $s$-Bennett linkage aligns at
\begin{equation}
  \label{eq:24}
  s' = \infty, \quad s'' = n_0.
\end{equation}
The most remarkable thing about Equations~\eqref{eq:23} and \eqref{eq:24} is
that that $t'$ and $t''$ do \emph{not depend} on $s$ and $s'$, $s''$ do
\emph{not depend} on $t$. Abstracting from our special geometric description to
the general case, we can thus state:

\begin{thm}
  \label{th:folded-Bennett}
  In a multi-Bennett 8R mechanism, the $t$-Bennett sub-mechanisms align precisely
  for two fixed parameter values $t'$, $t''$ and the $s$-Bennett sub-mechanisms
  align precisely for two fixed parameter values $s'$, $s''$. The points $(t',
  s')$, $(t', s'')$, $(t'', s')$, and $(t'', s'')$ in the configuration space
  correspond to the four aligned states of the complete mechanism, c.f.
  Theorem~\ref{th:folded}.
\end{thm}

Theorem~\ref{th:folded-Bennett} is illustrated in Figure~\ref{fig:8R}. There,
the eight links are visualized by cylinders around the common normals of
consecutive joint axes. This is clearly visible in the four totally aligned
configurations in the corners. The motions between neighbouring corners have $t
= t'$, $t = t''$, $s = s'$, or $s = s''$. Figure~\ref{fig:8R} also illustrates
the multi-Bennett's configuration space, a torus, and the four curves, meridian
and lateral circles on the torus, along which Bennett sub-mechanisms align.

\newcommand{\myR}[2][scale=0.075]{\includegraphics[#1]{./img/8R-#2.png}}
\begin{figure}
  \centering
  \begin{tblr}{colspec={X[c]X[c]X[c]X[c]},
      colsep=1mm,
    }
    \myR{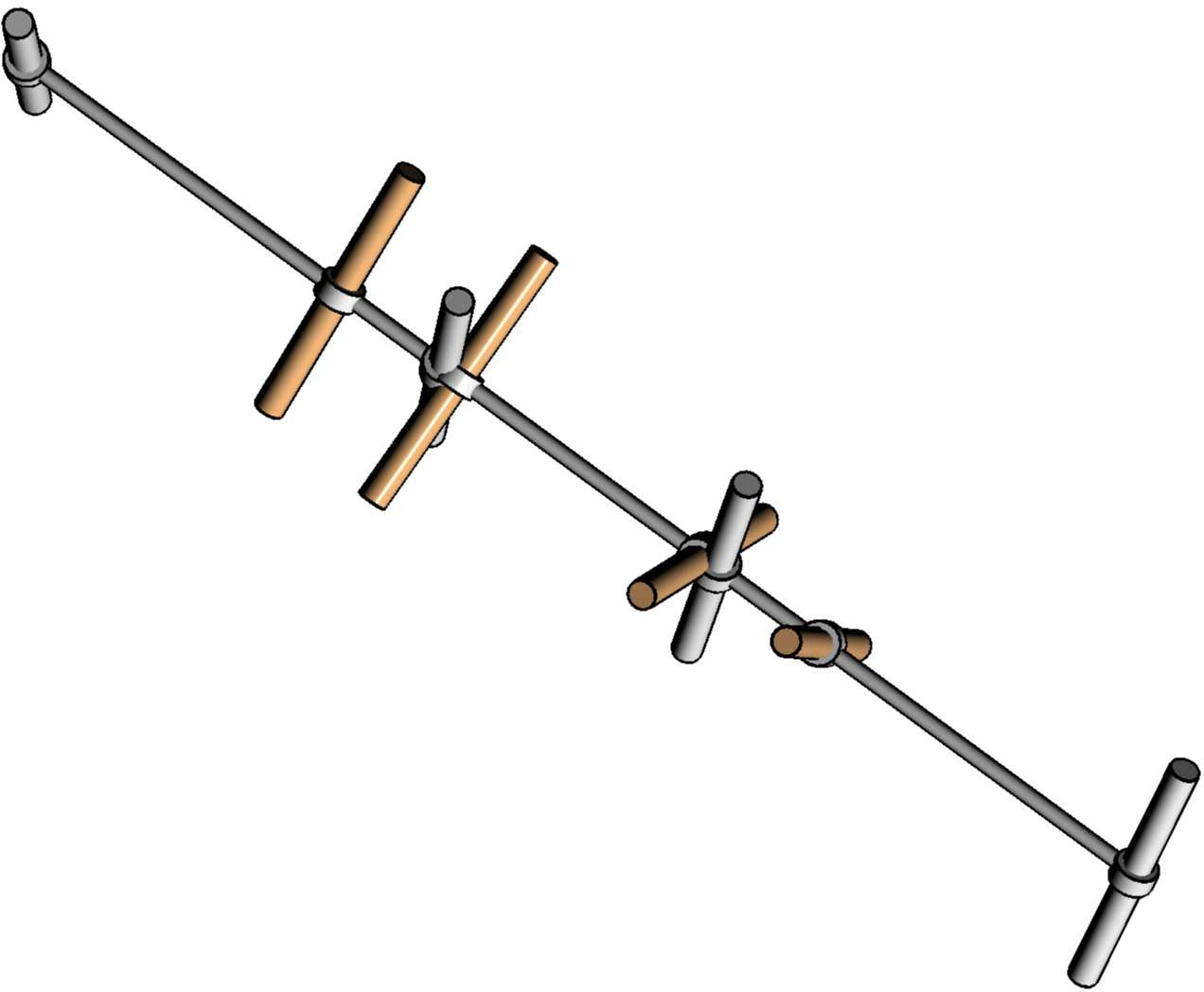} & \myR{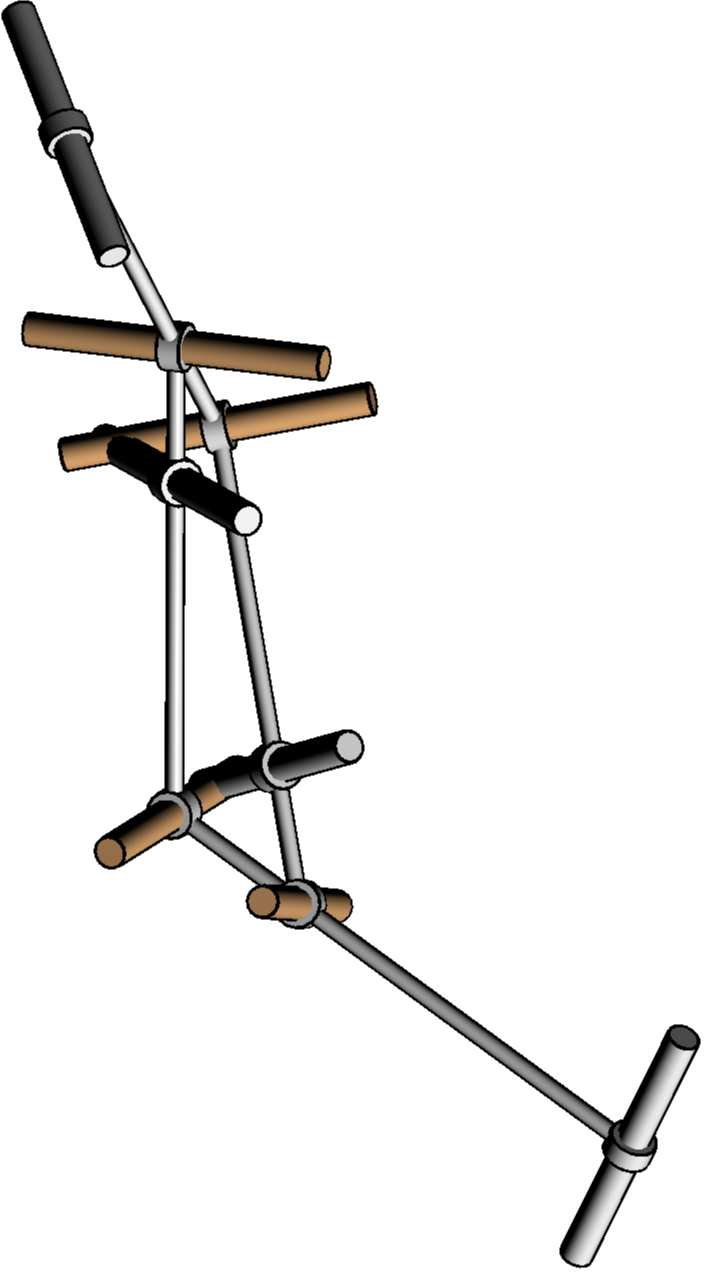}                                             & \myR{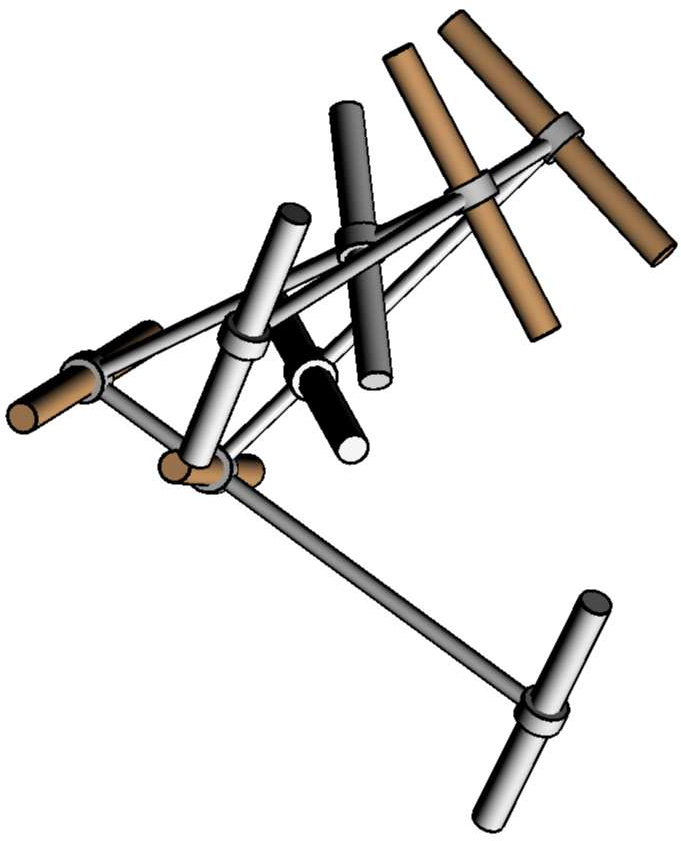} & \myR{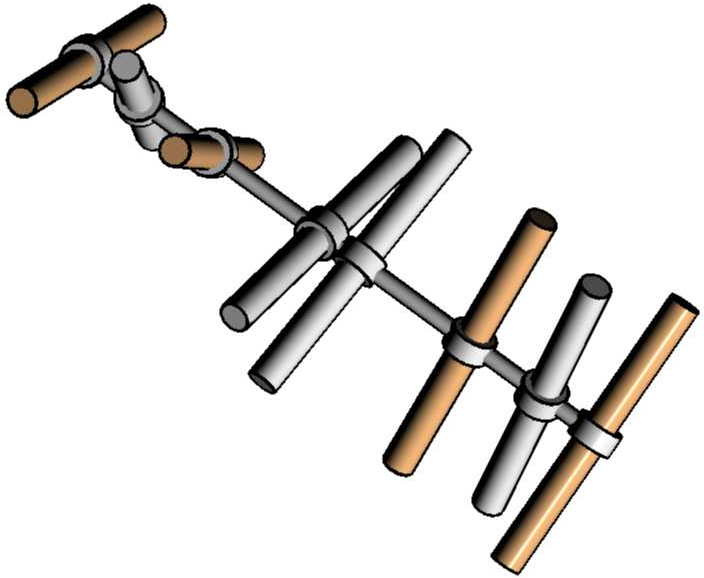} \\
    \myR{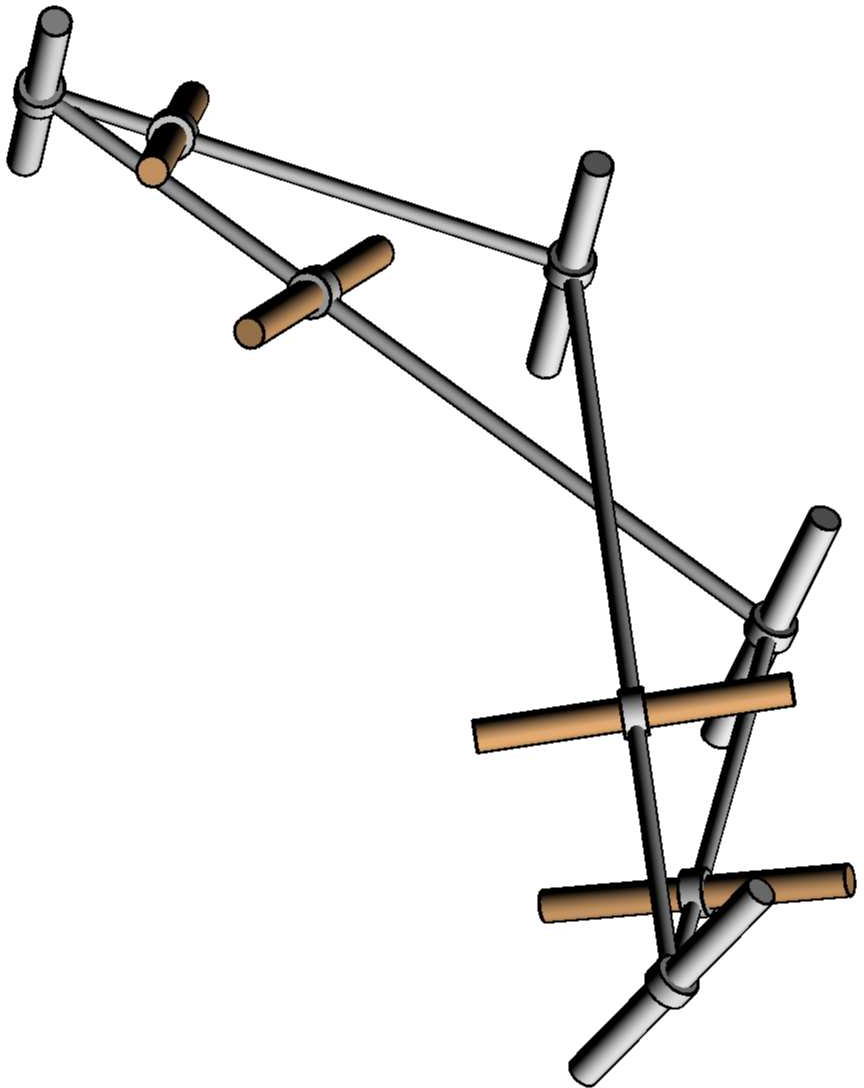} & \SetCell[r=2,c=2]{c} \includegraphics[]{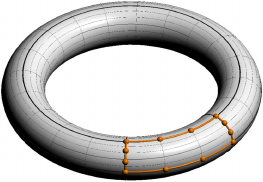} &          & \myR{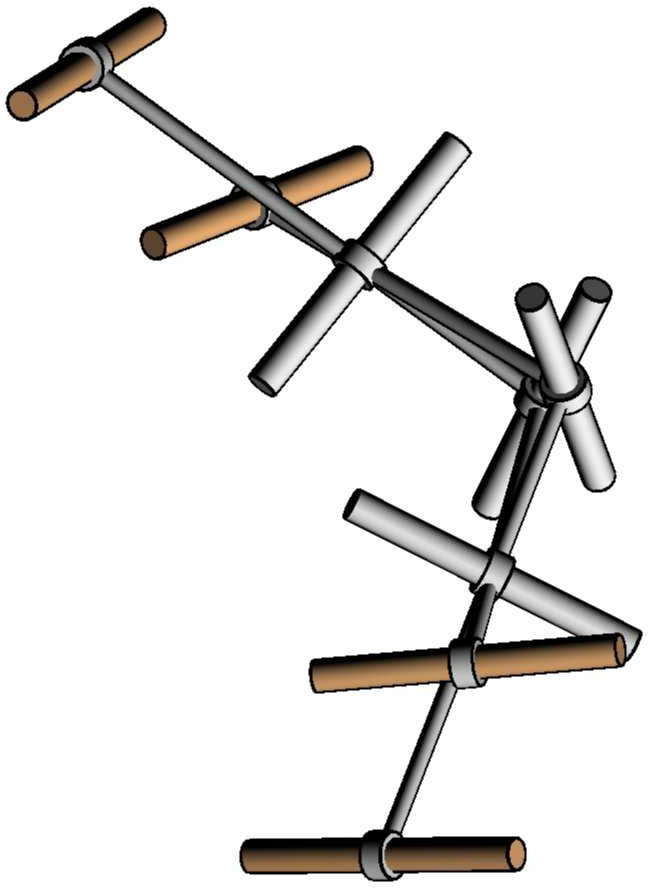}  \\
    \myR{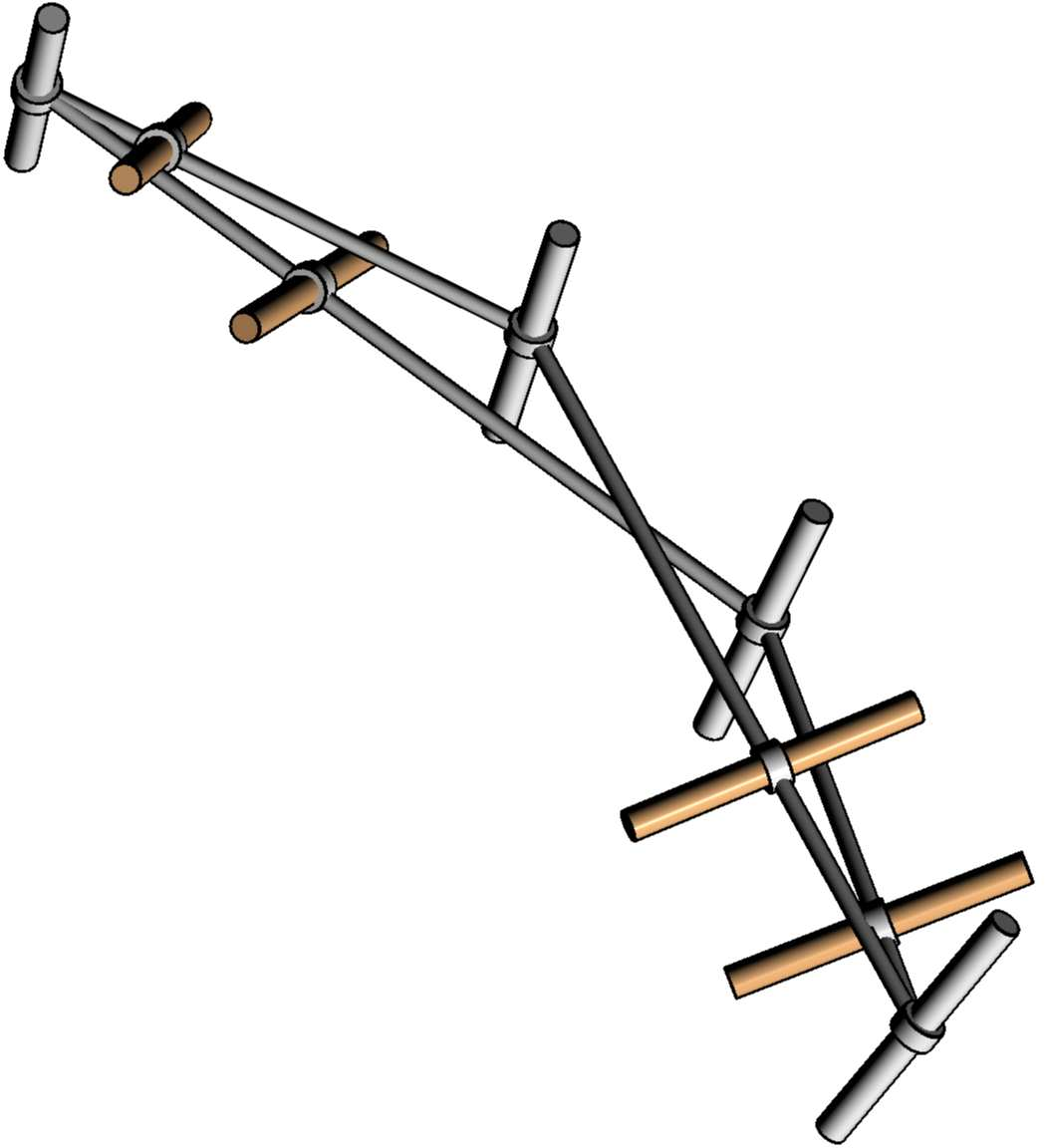} &                                                      &          & \myR{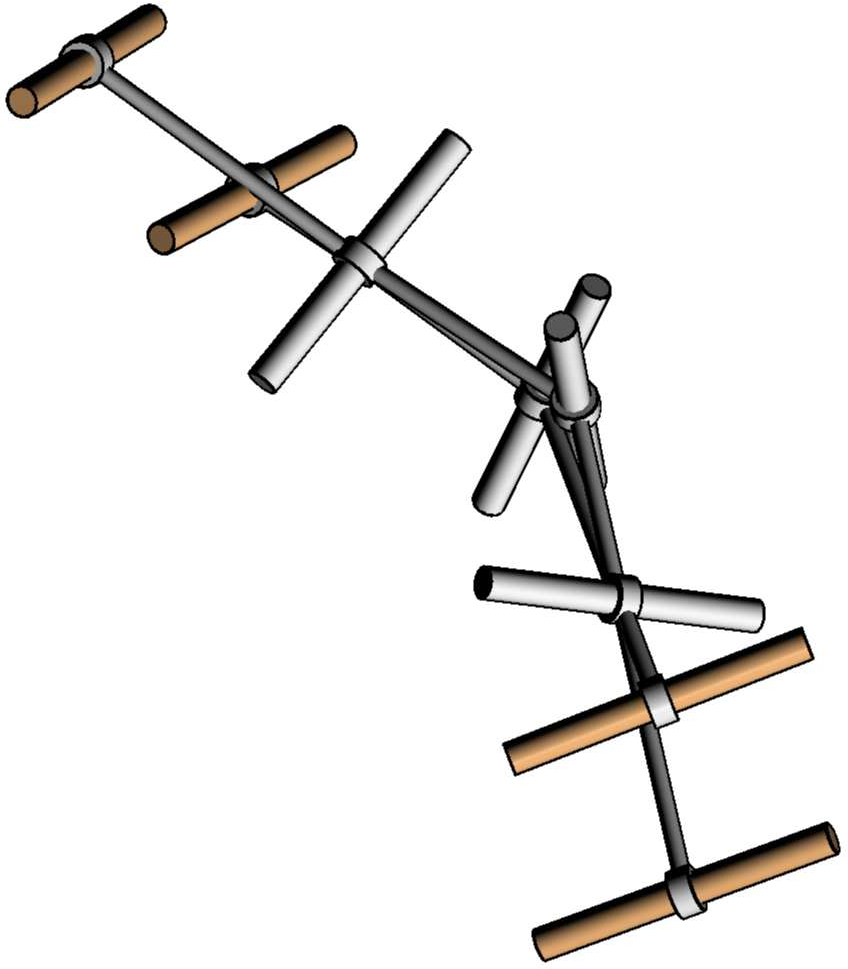}  \\
    \myR{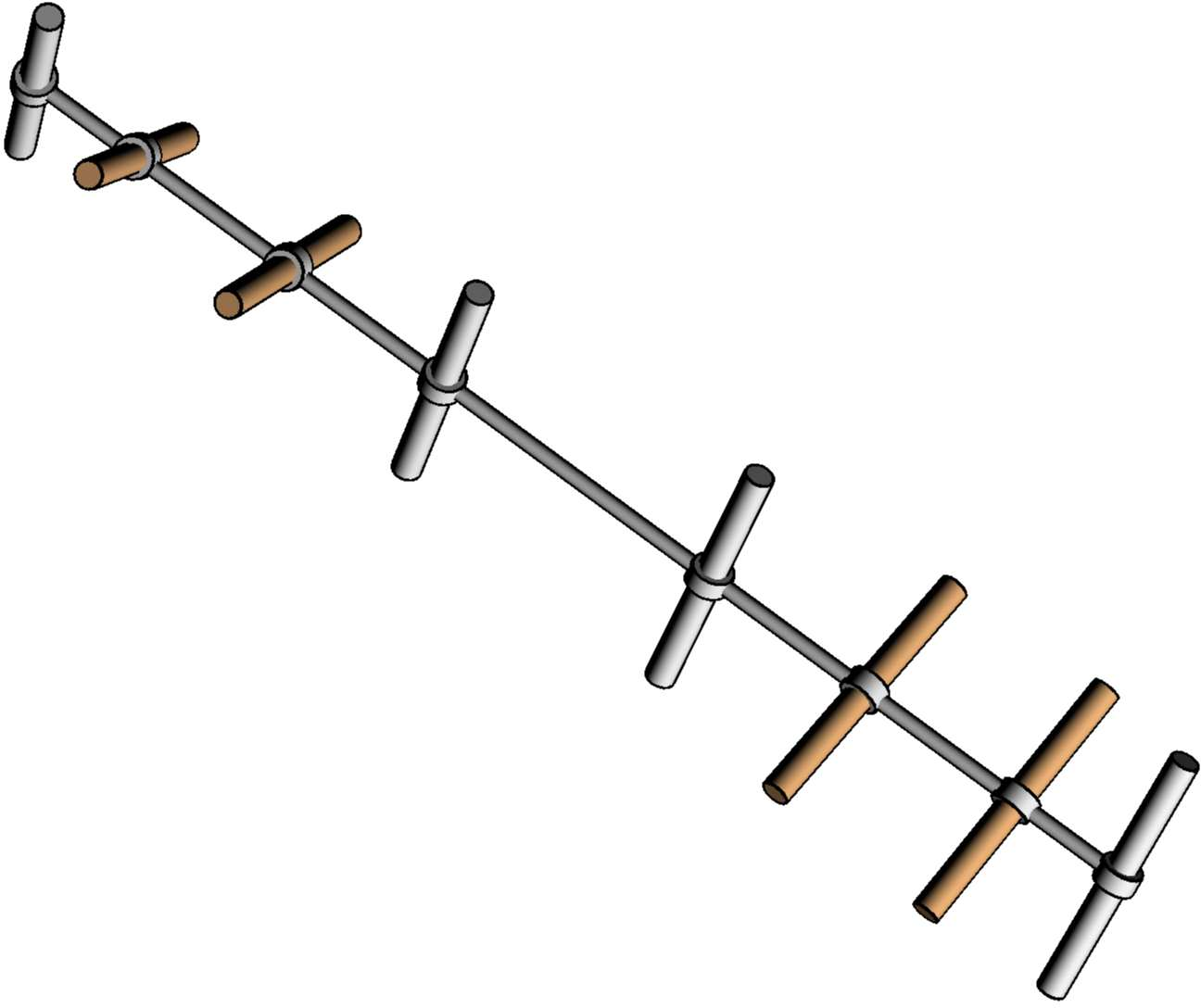} & \myR{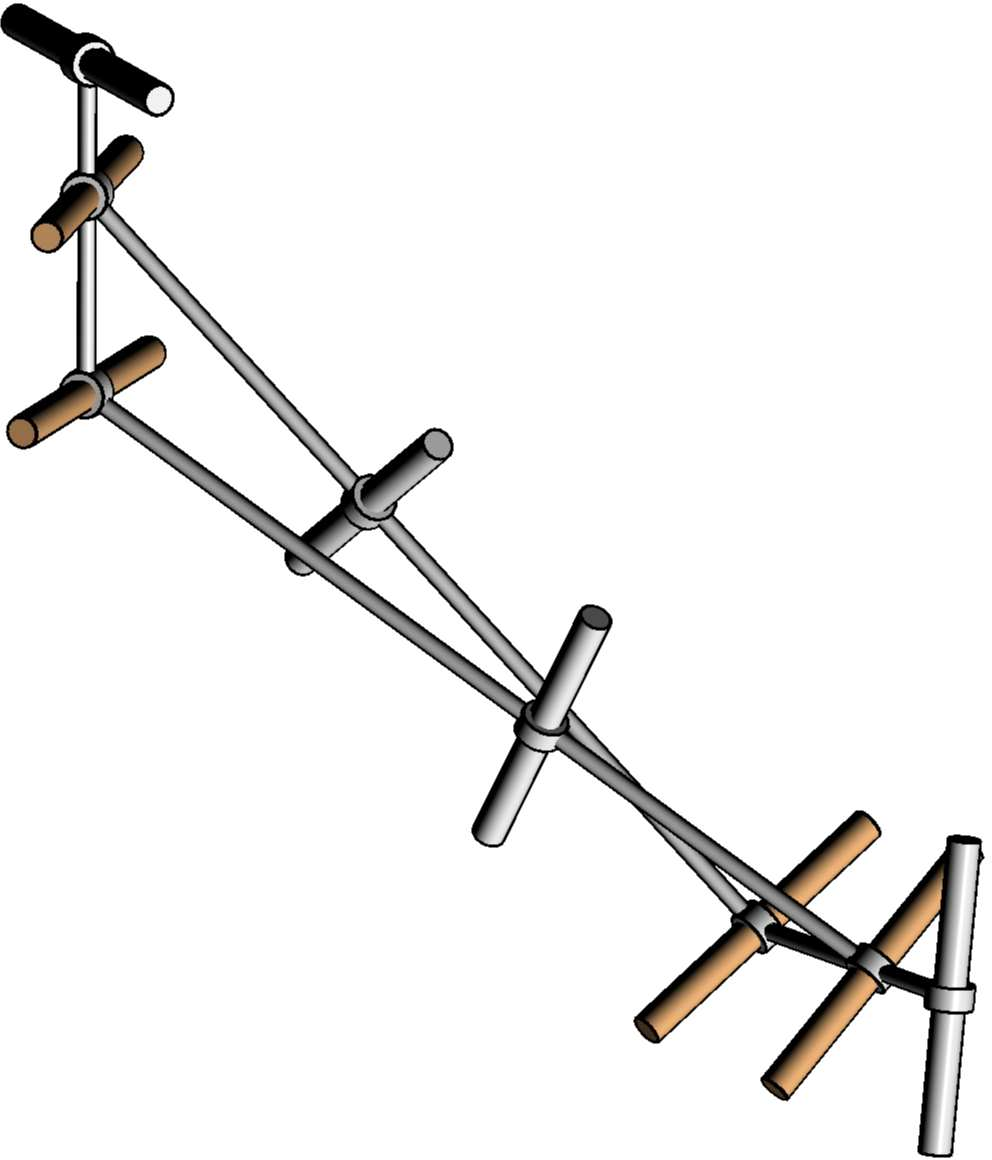}                                             & \myR{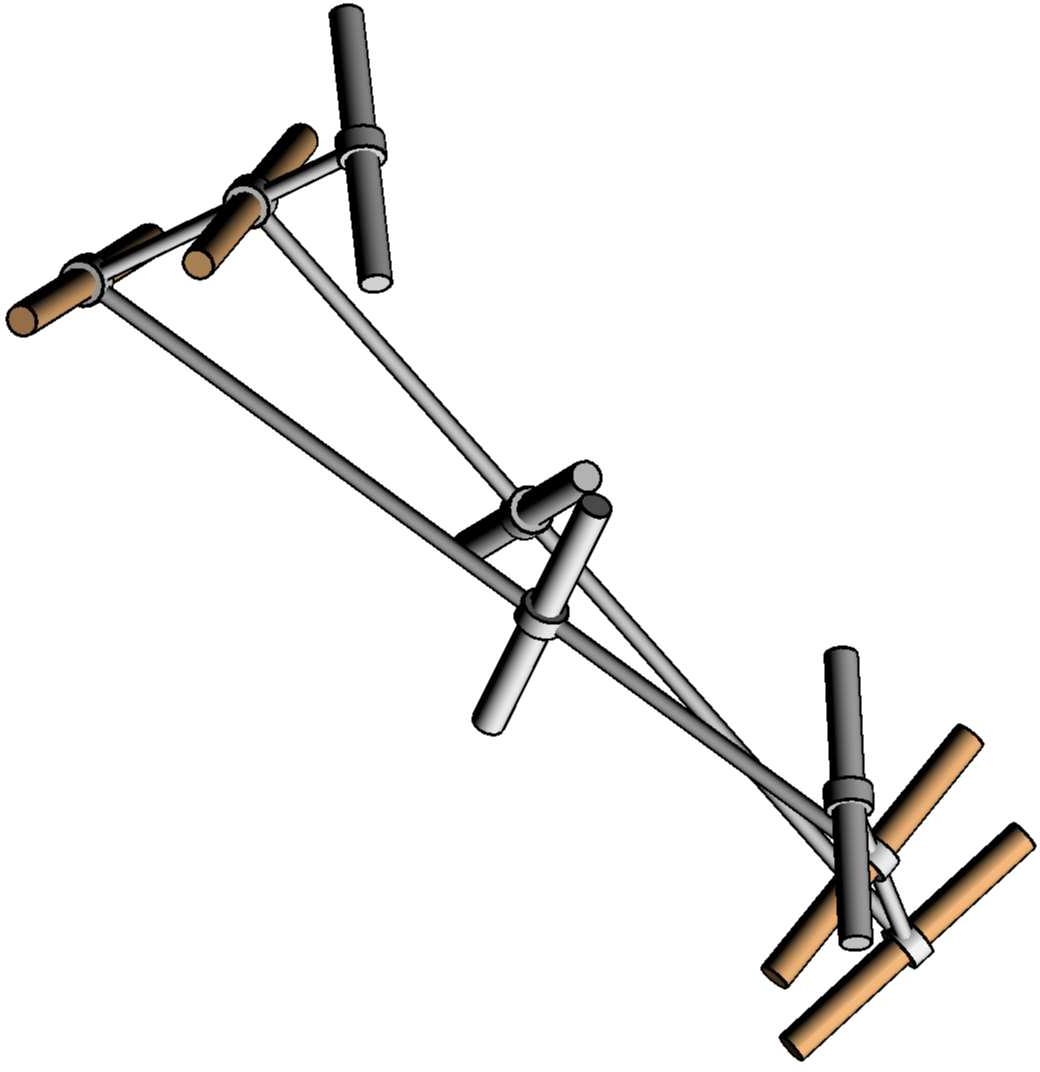} & \myR{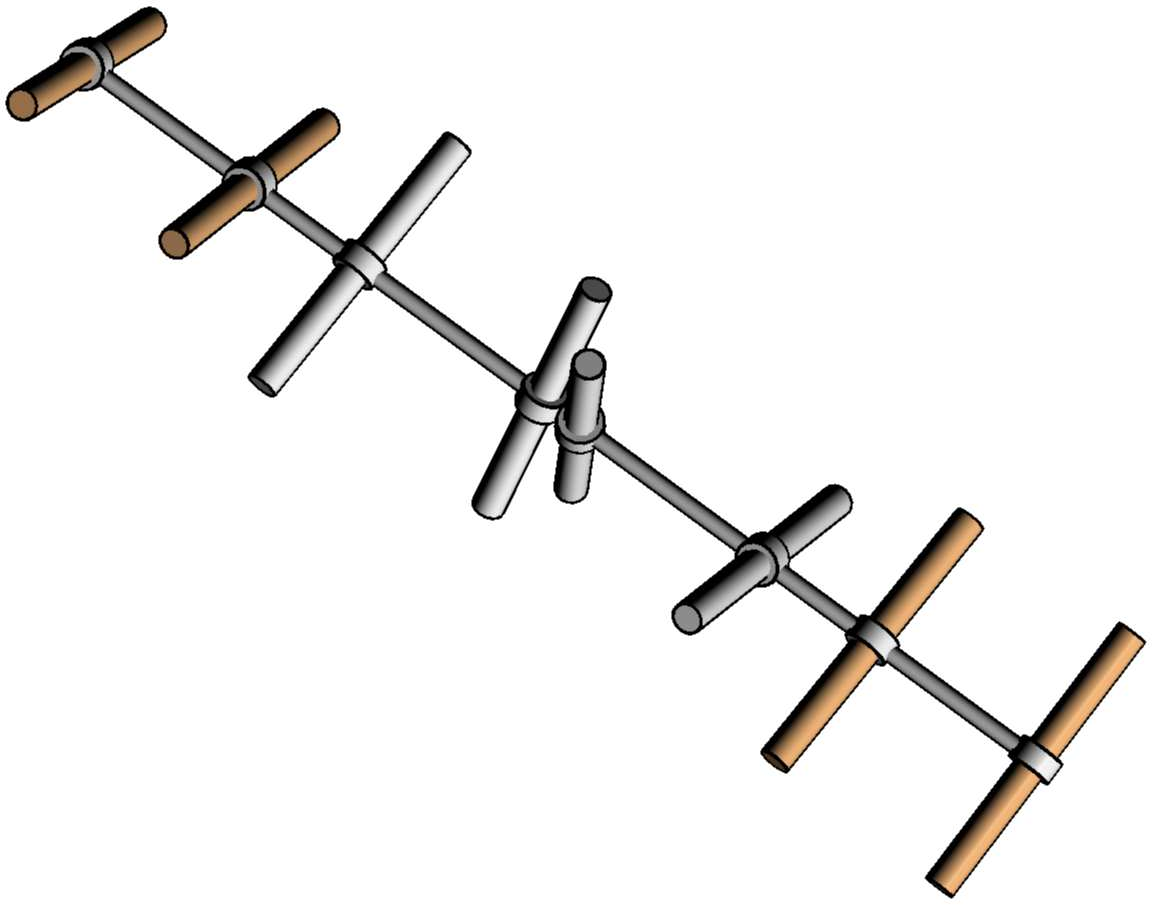}
  \end{tblr}
  \caption{Configuration space of multi-Bennett 8R mechanism and animated
    transitions between four totally aligned configurations in the corners;
    points correspond to depicted configurations.}
  \label{fig:8R}
\end{figure}

As expected, the Bennett ratio is not
constant within the family of $t$-Bennett mechanism but depends on $s$ (and vice
versa for $s$-Bennett mechanisms). However, a noteworthy property is:

\begin{thm}
  \label{th:bennett-ratio}
  The Bennett ratio within the family of $t$-Bennett linkages is a
  \emph{rational} function of degree four in $s$ and vice versa for the
  $s$-Bennett linkages.
\end{thm}

\begin{proof}
  A direct computation using computer algebra yields the value
  \begin{equation*}
    \tau = \frac{h_2m_2((h_2m_3+h_3m_2)^2(s^4 - 4n_0s^3) + c_2s^2 + c_1s + c_0)}{h_1\sqrt{h_2^2+h_3^2}\sqrt{m_2^2+m_3^2}(s^2 - 2n_0s + n_0^2 + n_2^2 + n_3^2)((h_3^2m_2^2-h_2^2m_3^2)s(s -2n_0) + d_0)}
  \end{equation*}
  for the $t$-Bennett ratio where
  \begin{multline*}
    c_2 = 6(h_2m_3+h_3m_2)^2n_0^2 + (2h_2^2m_3^2+2h_3^2m_2^2+4h_3^2m_3^2)n_2^2 \\
    + 2(2h_2^2m_2^2+h_2^2m_3^2+h_3^2m_2^2)n_3^2 - 4n_2n_3(h_2m_2-h_3m_3)(h_2m_3-h_3m_2),
  \end{multline*}
  \begin{multline*}
    c_1 = -4n_0^3(h_2m_3+h_3m_2)^2 - 4n_0n_2^2(h_2^2m_3^2+h_3^2m_2^2+2h_3^2m_3^2) \\
    + 8n_0n_2n_3(h_2m_2-h_3m_3)(h_2m_3-h_3m_2) - 4n_0n_3^2(2h_2^2m_2^2+h_2^2m_3^2+h_3^2m_2^2),
  \end{multline*}
  \begin{multline*}
    c_0 = n_0^4(h_2m_3+h_3m_2)^2
       + 2n_0^2n_2^2(h_2^2m_3^2+h_3^2m_2^2+2h_3^2m_3^2)\\
       + 4n_0^2n_2n_3(h_2m_2-h_3m_3)(h_3m_2-h_2m_3)
       + 2n_0^2n_3^2(2h_2^2m_2^2+h_2^2m_3^2+h_3^2m_2^2)
       + n_2^4(h_2m_3-h_3m_2)^2 \\
       + 4n_2^3n_3(h_2m_2+h_3m_3)(h_3m_2-h_2m_3)
       + 2n_2n_3^2(2h_2^2m_2^2-h_2^2m_3^2+6h_2h_3m_2m_3-h_3^2m_2^2+2h_3^2m_3^2)\\
       + 4n_2n_3^3(h_2m_2+h_3m_3)(h_2m_3-h_3m_2) + n_3^4(h_2m_3-h_3m_2)^2,
  \end{multline*}
  and
  \begin{multline*}
    d_0 = - 4h_2n_2n_3(2h_2m_2m_3-h_3m_2^2+h_3m_3^2) + n_0^2(h_3^2m_2^2-h_2^2m_3^2) \\
       + n_2^2(3h_2m_3-h_3m_2)(h_2m_3-h_3m_2) + n_3^2(4h_2^2m_2^2-h_2^2m_3^2+4h_2h_3m_2m_3+h_3^2m_2^2).
  \end{multline*}
  A similar formula can be derived for the $s$-Bennett ratio.
\end{proof}

\section{Conclusion and Future Research}
\label{sec:conclusion}

We presented the first example of a mechanism constructed from the factorization
of \emph{bivariate} motion polynomials and described some of its fundamental properties. Of course,
open questions remain.

The simple conditions on the mechanism's DH parameters which we describe in
Theorem~\ref{th:DH-relations} are necessary but, so we believe, not sufficient.
It would be desirable to augment them with further conditions to obtain a set of
sufficient conditions.

We further believe that the configuration space parametrized by the underlying
motion polynomial $C(s,t)$ is only a part of the mechanism's complete
configuration space. Obtaining a clearer picture of possible assembly modes or
bifurcations of the motion is certainly a worthy topic of future research.

The configuration space component described by $C(s,t)$ has many attractive
features for potential applications: It has a rational parametrization with low
degree parameter lines. The motion along a parameter line is the well-understood
coupler motion of a Bennett mechanism. Simple parametrization but also the
unusual separation into joints that only move with parameter $t$ and joints that
only move with parameter $s$ is expected to be beneficial for the control of a
multi-Bennett 8R mechanism. It can also be viewed as an adjustable Bennett
mechanism where rotation in one group of joints (say those parametrized by $s$)
changes the geometry of the $t$-Bennett mechanism.

\bibliographystyle{elsarticle-harv}

\end{document}